\documentclass[10pt,journal,compsoc]{IEEEtran}

\usepackage[nocompress]{cite}
\usepackage[pdftex]{graphicx}
\usepackage{epsfig}
\usepackage{amsmath}
\usepackage{amsthm}
\usepackage{mathtools}
\usepackage{amssymb}
\usepackage{mathrsfs}
\usepackage{bbm}
\usepackage{bbold}
\usepackage[ruled,noend]{algorithm2e}
\usepackage{float}
\usepackage{stfloats}
\usepackage{enumitem}
\usepackage[caption=false,font=footnotesize,labelfont=sf,textfont=sf]{subfig}
\usepackage{multirow}
\usepackage{url}
\usepackage{pifont}
\usepackage{colortbl}

\def\tailsize{\tau}
\def\reals{\mathbb{R}}
\newcommand{\xmark}{\ding{55}}
\long\def\comment#1{}
\long\def\commentout#1{}
\newtheorem{definition}{Definition}
\newtheorem{theorem}{Theorem}

\begin{document}
\title{A Review of Open-World Learning and Steps Toward Open-World Learning Without Labels}

\author{
Mohsen~Jafarzadeh,~\IEEEmembership{}
Akshay~Raj~Dhamija,~\IEEEmembership{}
Steve~Cruz,~\IEEEmembership{}
Chunchun~Li,~\IEEEmembership{}
Touqeer~Ahmad,~\IEEEmembership{}
Terrance~E.~Boult~\IEEEmembership{}

\thanks{M. Jafarzadeh, A.R. Dhamija, S. Cruz, C. Li, T. Ahmad, T. E. Boult are with the VAST lab, University of Colorado Colorado Springs, Colorado Springs, Colorado 80918, USA.}

\thanks{E-mail: tboult@vast.uccs.edu} 
}

\markboth{November~2021. ~A Review of Open-World Learning and Steps Toward Open-World Learning Without Labels}%
{Jafarzadeh \MakeLowercase{\textit{et al.}}: A Review of Open-World Learning and Steps Toward Open-World Learning Without Labels}

\IEEEtitleabstractindextext{%
\begin{abstract}
In open-world learning, an agent starts with a set of known classes, detects, and manages things that it does not know, and learns them over time from a non-stationary stream of data.   Open-world learning is related to but also distinct from a multitude of other learning problems and this paper briefly analyzes the key differences between a wide range of problems including incremental learning, generalized novelty discovery, and generalized zero-shot learning. 
This paper formalizes various open-world learning problems including open-world learning without labels. These open-world problems can be addressed with modifications to known elements, we present a new framework that enables agents to combine various modules for novelty-detection, novelty-characterization, incremental learning, and instance management to learn new classes from a stream of unlabeled data in an unsupervised manner, survey how to adapt a few state-of-the-art techniques to fit the framework and use them to define seven baselines for performance on the open-world learning without labels problem. We then discuss open-world learning quality and analyze how that can improve instance management. We also discuss some of the general ambiguity issues that occur in open-world learning without labels.
\end{abstract}

}

\maketitle

\IEEEdisplaynontitleabstractindextext

\IEEEpeerreviewmaketitle

\IEEEraisesectionheading{\section{Introduction}\label{sec:introduction}}

\IEEEPARstart{B}{abies} can detect novel objects and learn them even if they are not given semantic labels. Online vision-based systems, autonomous robots, and self-driving vehicles may confront new object classes in areas and must learn to deal with them even if they do not know the semantic label. These systems should first detect any new objects which are not in the training set. Then they should distinguish and learn new classes for these objects.  Also, they should recognize the new classes when they see them again. Ideally, each of the above steps should be done in an unsupervised manner.

\begin{figure*}[!t]
\centering
\includegraphics[width=.9\linewidth]{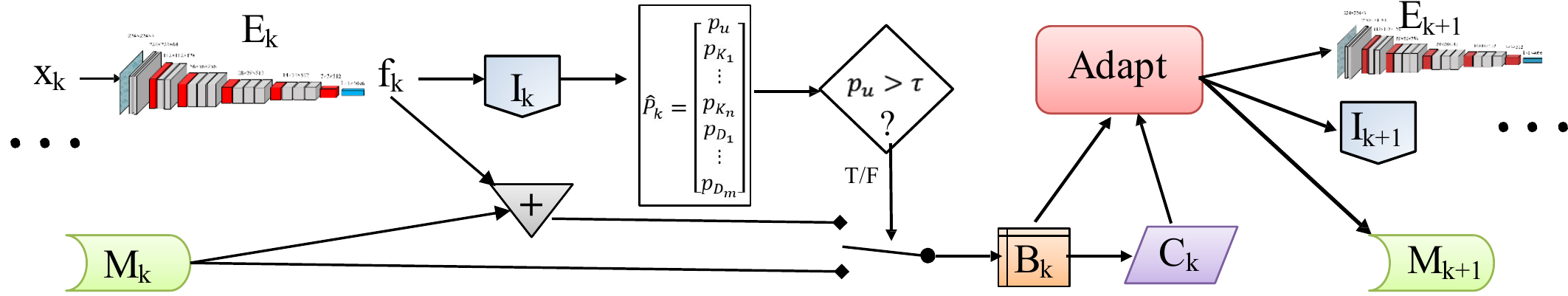}
\caption{Function block diagram of open-world learning. At time step $k$, agent $\mathcal{A}_k$ can be modeled with a memory $M_k$, a perception subsystem or feature extractor $E_k$, and a decision making or inference subsystem  $I_k$. The agent acts on open-world stream $\mathcal{S}^O$, see Eq. (\ref{eq_open_world_stream}). At time step $k$, feature extractor $E_k$ converts data $x_k \in \mathcal{S}^O$ to feature $f_k$. Then inference subsystem  $I_k$ predicts probabilities of data belonging to unknowns, knowns, and discovered classes, $P_k = [p_u \;\;\; p_{K_1} \;\;\; p_{K_2} \;\;\; \dots \;\;\; p_{K_{n-1}} \;\;\; p_{K_n} \;\;\; p_{D_1} \;\;\; p_{D_2} \;\;\; \dots \;\;\; p_{D_{{m_k}-1}} \;\;\; p_{D_{m_k}} ]^T$, where $n$ is the number of known classes in the training set and $m_k$ is the number of discovered classes. If probability of unknown $p_u$ is less than a threshold $\tau$, then the buffer $B_k$ is equal to the memory $M_k$, otherwise,  the buffer $B_k$ is equal to the concatenation of the feature $f_k$ and the memory $M_k$. Next, each instance of the buffer $B_k$, gets a label at function $C_K$ either supervised (human or other agents) or unsupervised via novelty discovery. Finally,  the agent $\mathcal{A}_k$ will be updated to $\mathcal{A}_{k+1}$ based on the buffer $B_k$ and the supervised/unsupervised labels $C_K$. The agent $\mathcal{A}_{k+1}$ will be used in  the next time step $k+1$.}
\label{fig_OWL}
\end{figure*}

While in \cite{bendale2015towards}, our lab provided the initial formal definition of open-world learning, that work and subsequent work on open-world vision \cite{rudd2017extreme} has supported detection of unknowns, but the learning of new items used fully supervised labeling of all new items.  In addition, prior work/formulations cannot directly address open-world learning without labels. Herein, we formalize a wide range of open-world related problems, survey components that might be used in solving them, and investigate open-world unsupervised class incremental learning of image classifiers for autonomous agents. Our motivation is to build fundamentals and formalize \textbf{open-world learning without labels} to be used along with other theories and solutions in the design of autonomous open-world learning robots in the future.

While few papers use the phrase open-world other senses, e.g to mean recognition with long-tail distributions\cite{liu2019large}, we follow in the spirit of \cite{bendale2015towards}, where operating in an open-world requires learning about previously unknown items, not just detecting or ignoring the unknown items.
Fig. \ref{fig_OWL} demonstrates a general cycle of open-world learning, a generalization of the model of open-world recognition presented in \cite{bendale2015towards}.
In open-world learning, agents start from an initial (potentially pre-trained) model. The agents confront a continuous stream of data that contains a mixture of known and unknown objects. The agent should (1) classify inputs into a known class or declare as unknown, (2) manage unknown instances to decide when there is enough data to learn, discover and distinguish between different classes of unknowns, (4) verify the quality of each cluster of unknowns, (5) manage and select the clusters of unknowns for generating pseudo-labels for adaption, and (6) learn newly qualified cluster of unknowns without forgetting previously learned classes. The major differences from prior work are the explicit step to manage which unknowns have sufficient information to support more processing/learning, and the explicit pseudo-labeling needed for the incremental learning without labels. This model also supports the potential to update feature representation in the adapt stage, although that aspect is not pursued in this paper.

While some recent surveys \cite{geng2020recent,boult2019learning} have looked at open-set and open-world learning, those works did not really differentiate these problems from the many related problems/algorithms which can be viewed as building blocks for open-world learning.  This paper analyzes over 100 related papers and puts them into context in terms of the key elements needed for open-world learning.  By looking at the properties needed vs provided by other researchers, our analysis also identifies gaps where there are opportunities for new research efforts.

The contributions of this paper include:
\begin{itemize}[topsep=0pt,itemsep=-1ex,partopsep=1ex,parsep=1ex]

\item Describing the difference between open-world learning and related fields, e.g., open-set recognition, generalized novelty discovery, incremental learning, generalized zero-shot learning, etc. See Table \ref{table_compare_problems};

\item Formalizing open-world learning without labels, also known as unsupervised open-world learning;

\item Presenting a general framework for building systems for OWL with labels;

\item Surveying state-of-the-art building blocks and proposing seven simple baseline agents that discover, characterize, and learn new classes without labels from an open-world stream of data by modifying Nearest Mean Classifier (NCM) \cite{guerriero2018ncm}, Nearest Non Outlier (NNO) \cite{bendale2015towards}, Gaussian Mixture Model (GMM) \cite{arandjelovic2005gmm}, Centroid-Based Concept Learning (CBCL) \cite{ayub2020cbcl}, Scaling Incremental Learning (SCAIL) \cite{belouadah2020scail}, and Extreme Value Machine (EVM) \cite{rudd2017extreme};

\item Extending our framework with adaptive management of unlabeled unknowns yielding a family of True Open-world Learners ;

\item Designing a new framework to evaluate performance in both supervised and unsupervised open-world scenarios;

\item Investigating metrics to measure the quality of open-world learning;

\item Showing that on ImageNet scale experiments, that a simple combination of standard components is statistically significantly weaker than the true open-world learners which include management of unknown items using our quality measure. 

\end{itemize}

\section{Open-World Learning Formalization}
\label{sec_formal}

While there has been growing work in open-set and open-world, few papers have formalized their problems and most have been inconsistent with the definitions of open-set in \cite{scheirer2012toward} or the open-world in \cite{bendale2015towards}, hence we believe there is value in offering new broader formal definitions and expanding them to a wide range of subclasses of work related to open-world learning. In the next section, we formalize multiple variants of open-world learning. In particular the definitions in  \cite{scheirer2012toward,bendale2015towards,rudd2017extreme} all required algorithms to have {\em provably bounded open-space risk}. In our refined definitions, we take that for granted as we prove, in the supplemental material, that for modern deep features from networks with L2 normalization, bounded inputs must produce bounded features, and hence any algorithms can generally claim bounded open-space risk. Furthermore and more practically, it has been suggested, see \cite{dhamija2018reducing}, that deep networks feature overlap is at least as big an issue, if not more of a problem, than long-distance features. In all other respects, our definitions are consistent, but we refine these earlier definitions and expand to new problems to address multiple variations of agents that engage in open-world learning.

We start the formalization by defining the different types of data: training set, known set, unknown set, generated set, auxiliary set, ignored set,  buffer set, and discovered set. Also, we define seen and unseen partitions of the aforementioned sets. Then we present a formalization of supervised, unsupervised, out-of-label, and out-of-distribution learning which are all well-known related but distinct problems. Then we define the open-world stream of data and the various types of open-world learning.   Examples of research in each category are provided in the next section and in rows in our table of related work, Table \ref{table_compare_problems}.

One of the complexities of open-world learning is that over time data that was initially unknown can be labeled and become known, so we must carefully define the data sets as a function of time.   Let us define the $m$'th class of a known set at time step $k \in \mathbb{W}$ as $k_{k,m}$,  $n$'th class unknown at time $k$ as $u_{k,n}$,  $h$'th class generated with $g_{k,h}$,  $a$'th class auxiliary with $x_{k,a}$, $i$'th class of training set with $t_{k,i}$, $s$'th class detected as novel and kept in buffer with $b_{k,s}$, $r$'th class detected as novel but ignored set with $i_{k,r}$, and $t$'th class discovered with $d_{k,t}$. The ignored set, buffer set, and discovered set are items used in various algorithms/settings and are derivative of known and unknown sets, i.e., they do not have an independent identity. Some known data points can be incorrectly predicted as unknown.  The known, unknown, generated, and auxiliary sets can be partitioned into seen and unseen sets, i.e., $k_{k,m} = k^{seen}_{k,m} \cup k^{unseen}_{k,m}$, $u_{k,n} = u^{seen}_{k,n} \cup u^{unseen}_{k,n}$, $g_{k,h} = g^{seen}_{k,h} \cup g^{unseen}_{k,h}$, and $x_{k,a} = x^{seen}_{k,a} \cup x^{unseen}_{k,a}$.   The ignored set, buffer set, and discovered set do not have unseen elements.

By getting union at time step $k$, we can define training set $T_k = \cup_1^{t_k} t_{k,m}$, known set $K_k = \cup_1^{m_k} k_{k,m}$, unknown set $U_k = \cup_1^{n_k} u_{k,n}$, generated set $G_k = \cup_1^{h_k} g_{k,h}$, auxiliary set $X_k = \cup_1^{a_k} x_{k,a}$, ignored set $I_k = \cup_1^{i_k} i_{k,i}$, buffer set $B_k = \cup_1^{b_k} b_{k,s}$, and discovered set $D_k = \cup_1^{t_k} d_{k,t}$.  Similarly, the four sets can be partitioned into seen and unseen sets, i.e.,  $K_{k} = K^{seen}_{k} \cup K^{unseen}_{k}$, $U_{k} = U^{seen}_{k} \cup U^{unseen}_{k}$, $G_{k} = G^{seen}_{k} \cup G^{unseen}_{k}$, and $X_{k} = X^{seen}_{k} \cup X^{unseen}_{k}$.  The discovered set $D_k$ can be viewed as a subset of either known or unknown. The world is $\mathcal{W}_k = K_{k} \cup U_{k}$. The universe is $\mathcal{U}_k = K_{k} \cup U_{k} \cup X_{k} \cup G_{k} \cup X_{k}$. Both word and universe can be partitioned to seen and unseen, i.e.,  $\mathcal{W}_{k} = \mathcal{W}^{seen}_{k} \cup \mathcal{W}^{unseen}_{k}$ and  $\mathcal{U}_{k} = \mathcal{U}^{seen}_{k} \cup \mathcal{U}^{unseen}_{k}$.

Let us show human labeling function with $\ell(.)$, generated labeling function with $\breve{\ell}(.)$, and pseudo labeling function with $\hat{\ell}(.)$. A data point in the training set $x \in T_k$ can have human label $y = \ell(x)$,  generated label $\breve{y} = \breve{\ell}(x)$, or pseudo label $\hat{y} = \hat{\ell}(x)$. Let us define  $L_k = \{ \ell(x) | x \in T_k \}$,  $\breve{L}_k = \{ \ell(x) | x \in T_k \}$, and  $\hat{L}_k = \{ \ell(x) | x \in T_k \}$. If $L_k$ is given, we call it supervised learning or learning with label. If both $L_k$ and $\breve{L}$ is given, we call it generative supervised learning. If for $k \geq 1$, $L_k$ is hidden and only agent has access to $\hat{L}_k$, we call it unsupervised learning, self-supervised learning, or learning with no label. The case that $L_k$ is hidden for all $L_k$ is fully unsupervised learning , which is out of scope of this paper.
Time step $k=0$ is  called pre-training, warm up, base session, or pre-streaming. If $\forall k \in \mathbb{N}  \; \forall x \in T_k \exists n \in \mathbb{N} \; \textup{s. t.} \;  x \in U_{k-n} $, we call it learning using out-of-label. In this paper, we focus on the more commonly considered out-of-distribution, i.e., out-of-label is beyond scope of this paper.
 
\begin{definition}{Open-world stream}
\label{def_stream}
\\ Stream is a time series that each member is a single or batch of data point(s) drawn from the query set $Q_k$ at time step $k\geq 1$, i.e., 
 \begin{equation}
\mathcal{S} = \{ S_k \subset Q_k \;\; | \;\;   \forall k \in \mathbb{N} \} 
\end{equation}
 \begin{equation}
\mathcal{S}_k = \{ x_{k,j} \in Q_k  \;\; | \;\;    1 \leq j \leq j_k,  j \in \mathbb{N}, j_k \in \mathbb{N}  \} 
\end{equation}
The base session or pre-training, i.e., $k=0$, does not belong to the stream. If $j_k = 1$ for all $k$, then it is a data point stream, otherwise, it is batch stream. Generally, there is not any condition that $j_k = j_{k+1}$ or $j_k \neq j_{k+1}$. 
The closed-set stream is a time series of both seen and unseen drawn from pre-training known session classes $K_0$:
\begin{equation}
\mathcal{S}_k^C = \{ x_{k,j} \in Q_k \;\; |  \;\; Q_k \subseteq K_0\} 
\end{equation}
The out-of-label stream is a time series of both seen and unseen that draw from pre-training word session classes $\mathcal{W}_0$:
\begin{equation}
\mathcal{S}_k^L = \{ x_{k,j} \in Q_k \;\; |  \;\; Q_k \subseteq \mathcal{W}_0 \}
\end{equation}
The open-world stream is a time series of both seen and unseen (usually unseen) that 
\begin{equation}
\label{eq_open_world_stream}
\mathcal{S}_k^O = \{ x_{k,j} \in Q_k | \;\; Q_k \subseteq \mathcal{W}_k  \}
\end{equation}
The incremental label stream is a time series of label of past time steps
\begin{equation}
\mathcal{L}_k^I = \{  y_{k-1,j}   \; \; \forall j \leq n_{k-1} \;\; |  \;\; y_{k-1,j} = \ell(x_{k-1,j})  \} 
\end{equation}
If $\forall (k, \tau) \;\; \mathcal{L}_k^I \cap \mathcal{L}_{k-\tau}^I = \phi$, then the stream is called curriculum incremental label stream.
The active label stream is a time series of label of past time steps requested by agent
\begin{equation}
\mathcal{L}_k^A = \{  y_{k-1,j}   \;\; |  \;\;  \forall j \in H_{k-1}  \} 
\end{equation}
where $H$ is set of requested label. Indeed, $H$ can be or not be curriculum and budgeted. If $\forall (k, \tau) \;\; \mathcal{L}_k^A \cap \mathcal{L}_{k-\tau}^A = \phi$, then the stream is called curriculum active label stream.
The generated label stream is a time series of label of the current step 
\begin{equation}
\mathcal{L}_k^G = \{  y_{k,j}   \;\; |  \;\;  y_{k,j} = \breve{\ell}(x_{k,j})   \;\;  \forall x_{k,j} \in  G_k \} 
\end{equation}
\end{definition}

\begin{definition}{Novelty detector}
\label{def_novelty_detection}
\\ Novelty detector is a binary classifier inside the agent $\mathcal{A}_k$ at the time step $k$ that decides if each input $x_{k,j} \in \mathcal{S}_k$ is known $x_{k,j} \in K_k$  or unknown $x_{k,j} \in U_k$.
\end{definition}

\begin{definition}{Novelty discovery}
\label{def_novelty_discovery}
\\ Novelty discovery is a subsystem of the agent $\mathcal{A}_k$ at the time step $k$ that receives data points that predicte novel by the novelty detector and group the similar points together. The novelty discovery may or may not access to attribute, rules, or external knowledge.
\end{definition}

\begin{definition}{Novelty manager}
\label{def_novelty_manager}
\\ Novelty manager is a subsystem of the agent $\mathcal{A}_k$ at the time step $k$ that receives data points that are predicted by a novel novelty detector and decides to move it to ignore set $I_k$ or buffer set $B_k$. Then it decides to send the buffer set to a novelty discovery subsystem and with what hyper parameters values. Next, It receives the output of novelty discovery subsystem and applies quality analysis. Finally, based on the computed quality, it will decide each cluster to move to discovered set $D_k$, return it back to buffer $B_k$, or put it in ignore set $I_k$.
\end{definition}

\begin{definition}{Open-set recognizer}
\label{def_OSR}
\\ The agent $\mathcal{A}_k$ is an open-set recognizer if it acts on open-world streams and maps each input $x_{k,j} \in \mathcal{W}_k$ to a vector of probabilities of $x_k$, belonging to one of the $m_0$ known classes, $k_1 \ldots k_{m_0}$ or the reject class $I$ in the stream. Because the agent does not learn new classes, $K_k = K_0 \;\; \forall k \in \mathbb{N}$.
\end{definition}

\begin{definition}{Open-set representation incremental learner}
\label{def_OS_RIL}
\\ The agent $\mathcal{A}_k: \mathcal{W}_k \mapsto \reals^{1+m_0} $ is an supervised open-set representation learner if it acts on open-world streams and  incremental label stream $\mathcal{L}_k^I$ and maps each input $x_{k,j} \in \mathcal{W}_k$ to a vector of probabilities of $x_{k,j}$ belonging to one of the $m_0$ known classes $k_1 \ldots k_{m_0}$ or the reject class $I$ in the stream. If the incremental label stream contains classes that are not in the pre-training set $K_0$, it assumes the label is equal to reject class $I$. Because the agent does not learn new classes, $K_k = K_0  \;\; \forall k \in \mathbb{N}$. If an agent does not use incremental label stream, it is an unsupervised open-set representation learner.
\end{definition}

\begin{definition}{Open-set class incremental learner}
\label{def_OS_CIL}
\\ The agent $\mathcal{A}_k: \mathcal{W}_k \mapsto \reals^{1+m_k} $ is an  open-set class incremental learner if it acts on open-world streams and incremental label stream and maps each input $x_{k,j} \in \mathcal{W}_k$ to a vector of probabilities of $x_{k,j}$ belonging to one of the reject class $I$, $m_k$ known classes $k_1 \ldots k_{m_k}$ in the stream. Then it learns new classes and may update existing classes based on incremental label stream $\mathcal{L}^I_k$ without learning new representation. Let us decouple the agent $\mathcal{A}_k$ in to representation extractor $R_k: \mathcal{W}_k \mapsto \reals^{f}$ and class predictor $C_k: \reals^{f} \mapsto \reals^{1+m_k}$ where $f \in \mathbb{N}$ is constant. The agent keeps $R_k$ frozen and updates $C_k$, i.e, $R_k(x) = R_0(x)$.
\end{definition}

\begin{definition}{Open-set class and representation incremental learner}
\label{def_OS_RCIL}
\\ The open-set class and representation incremental learner is similar to the open-set class incremental learner, but the difference is that the agent updates both $R_k$ and $C_k$ simultaneously.
\end{definition}

\begin{definition}{Unsupervised representation incremental learner}
\label{def_URIL}
\\ The agent $\mathcal{A}_k: \mathcal{W}_k \mapsto \reals^{m_0}$ is an unsupervised representation learner if it acts on open-world streams without label stream and maps each input $x_{k,j} \in \mathcal{W}_k$ to a vector of probabilities of $x_{k,j}$ belonging to one of the $m_0$ pre-training classes $k_1 \ldots K_{m_0}$.  Because the agent does not learn new classes, $K_k = K_0  \;\; \forall k \in \mathbb{N}$.
\end{definition}

\begin{definition}{Unsupervised class incremental learner}
\label{def_UCIL}
\\ The agent $\mathcal{A}_k: \mathcal{W}_k \mapsto \reals^{1+m_k} $ is an unsupervised class learner if it acts on open-world streams without label stream and maps each input $x_{k,j} \in \mathcal{W}_k$ to a vector of probabilities of $x_{k,j}$ belonging to one of the $m_k$ known classes $k_1 \ldots k_{m_k}$ or a new class learn in the stream. Then it learns the new class and may update existing classes without learning new representation. The agent keeps $R_k$ frozen and updates $C_k$, i.e, $R_k(x) = R_0(x)$.
\end{definition}

\begin{definition}{Unsupervised class and representation incremental learner}
\label{def_UCRIL}
\\ The unsupervised class and incremental learner is similar to the unsupervised class learner but the difference is that the agent updates both $R_k$ and $C_k$ simultaneously.
\end{definition}

\begin{definition}{Open-world representation learner}
\label{def_OWL_R}
\\ The open-world representation incremental learner $\mathcal{A}_k: \mathcal{W}_k \mapsto \reals^{1 + m_0}$ is similar to the open-set representation learner, but it uses active label stream $\mathcal{L}_k^A$ instead of incremental label stream $\mathcal{L}_k^I$. The active label stream may or may not be budgeted. 
\end{definition}

\begin{definition}{Open-world class learner}
\label{def_OWL_C}
\\ The open-world class learner $\mathcal{A}_k: \mathcal{W}_k \mapsto \reals^{1 + m_k}$  is similar to the open-set class incremental learner, but it uses active label stream $\mathcal{L}_k^A$ instead of incremental label stream $\mathcal{L}_k^I$. Using the active label stream causes the number of predicted classes ($m_k$) to not track the number of expected classes.
\end{definition}

\begin{definition}{Open-world class and representation learner}
\label{def_OWL_CR}
\\ The open-world class and incremental learner is similar to the open-world class learner but the difference is that the agent updates both $R_k$ and $C_k$ simultaneously.
\end{definition}

\begin{definition}{Unsupervised open-world representation learner}
\label{def_UOWL_R}
\\ The agent $\mathcal{A}_k: \mathcal{W}_k \mapsto \reals^{1+m_0} $ is an unsupervised open-world representation learner if it acts on open-world streams without label stream  and maps each input $x_{k,j} \in \mathcal{W}_k$ to a vector of probabilities of $x_{k,j}$ belonging to one of the $m_0$ known classes $k_1 \ldots k_{m_0}$ or the reject class $I$ in the stream. Because the agent does not learn new classes, $K_k = K_0  \;\; \forall n \in \mathbb{N}$. The main difference with unsupervised open-set representation incremental learner is that it has data manager. 
\end{definition}

\begin{definition}{Unsupervised open-world class learner}
\label{def_UOWL_C}
\\ The agent $\mathcal{A}_k: \mathcal{W}_k \mapsto \reals^{1+m_k} $ is an unsupervised open-world class learner if it acts on open-world streams and without label stream and maps each input $x_{k,j} \in \mathcal{W}_k$ to a vector of probabilities of $x_{k,j}$ belonging to one of the the reject class $I$, $m_k$ known classes $k_1 \ldots k_{m_k}$ in the stream. Then it learns new classes and may update existing classes based on novelty discover without learning new representation. The agent keeps $R_k$ frozen and updates $C_k$, i.e, $R_k(x) = R_0(x)$. The main difference with unsupervised open-set representation learner is that it has data manager. 
\end{definition}

\begin{definition}{Unsupervised open-world class and representation learner}
\label{def_UOWL_CR}
\\ The  unsupervised open-world class and representation learner  $\mathcal{A}_k: \mathcal{W}_k \mapsto \reals^{1+m_k} $ is similar to the unsupervised open-world class learner, but it updates both $R_k$ and $C_k$.  The main difference with unsupervised open-set class and representation learner is that it has data manager. 
\end{definition}

\begin{definition}{Zero-shot learner}
\label{def_zero_shot}
\\ Any above defined agent is a zero-shot learner if, in addition to representation extractor $R_k$, it uses both attributes and rules for classifying existing class, detecting novel inputs, discovering new classes, or learning new classes. Attributes are properties of classes, for example, color, texture, number of legs. Rules are texts, tables, or dictionaries that describe classes. The human language dictionary is the best example of rules. Although direct rules are more convenient, rules can be indirect finding items via an encyclopedia, news, or entire web.  
\end{definition}

\begin{definition}{Single-shot learner}
\label{def_few_shot}
\\ Any above defined agent is a single-shot learner if they learn new classes or updates representation with only a data point.
\end{definition}

\begin{definition}{Few-shot learner}
\label{def_single_shot}
\\ Any above defined agent is a few-shot learner if they learn new classes or updates representation with a few data points, for example up to ten points but usually five points.
\end{definition}

\def\xmark{\ding{55}}
\def\xmark{$\cdot$}
\begin{table*}[!t]
\caption{Key properties of open-world learning and related research areas are columns where different research areas with references are shown in rows. For compactness `IL' stands for incremental learning and `OWL' for open-world learning. \checkmark means a research area provides that specific property;  a dot(\xmark)  means it does not provide that property.  }
{\small
\begin{center}
\begin{tabular}{|c|c|c|c|c|c|c|c|c|c|} 
\hline 
Research Area \\ \& Related papers  & \rotatebox[origin=b]{90}{\parbox{45mm}{\centering  Attributes \& rules unnecessary}}  & \rotatebox[origin=b]{90}{\parbox{45mm}{\centering Classifying  known}} &  \rotatebox[origin=b]{90}{\parbox{45mm}{\centering Detecting  unknowns}}  & \rotatebox[origin=b]{90}{\parbox{45mm}{\centering Managing unknowns}}   &  \rotatebox[origin=b]{90}{\parbox{40mm}{\centering Clustering unknowns}}   & \rotatebox[origin=b]{90}{\parbox{45mm}{\centering Multi-round classes learning}}   &  \rotatebox[origin=b]{90}{ \parbox{45mm}{\centering Multi-round  rep.  learning }}  &  \rotatebox[origin=b]{90}{ \parbox{45mm}{\centering Multi-round learning without label}} &  \rotatebox[origin=b]{90}{ \parbox{45mm}{\centering Demonstrated scalability}}   \\
\hline
OOD detection \cite{hendrycks17baseline,liang2018enhancing, vyas2018out, Lee2020gradientsNN, techapanurak2020hyperparameter, liu2020energy, sehwag2019analyzing, hsu2020generalized, sastry2020detecting, techapanurak2020hyperparameter, mohseni2020self}, & \multirow{3}{*}{\checkmark} & \multirow{3}{*}{\xmark} & \multirow{3}{*}{\checkmark} & \multirow{3}{*}{\xmark} & \multirow{3}{*}{\xmark} & \multirow{3}{*}{\xmark} & \multirow{3}{*}{\xmark} & \multirow{3}{*}{\xmark} & \multirow{3}{*}{\checkmark}  \\
anomaly detection \cite{hendrycks2018deep, golan2018deep, bergman2019classification, ruff2019deep, kimura2020adversarial, nguyen2019anomaly, zenati2018adversarially, li2021deep, fan2020robust, bergmann2020uninformed}, & & & & & & & &  & \\
and  novelty  detection \cite{abati2019latent, perera2019deep, oza2020utilizing, tack2020csi, zhang2020multi, lee2018hierarchical, schultheiss2017finding, oza2020multiple, bhattacharjee2020multi}  & & & & & & & & & \\
\hline
Open-set recognition \cite{scheirer2012toward, dhamija2018reducing, yoshihashi2019classification, oza2019c2ae, miller2021class, sun2020conditional, liu2019large, geng2020recent, cevikalp2019polyhedral}& \checkmark  &  \checkmark &  \checkmark &    \xmark   &   \xmark &   \xmark &   \xmark  &   \xmark  & \checkmark  \\ 
\hline
Novelty discovery  \cite{cho2015unsupervised, hsu2018learning, han2019learning, vo2019unsupervised, vo2020toward, wei2019unsupervised, qing2021end}  &  \checkmark &  \checkmark & \xmark    &   \xmark   &   \checkmark &  \xmark  &    &   \xmark  & \checkmark  \\ 
\hline
Generalized Novelty discovery \cite{han2019autonovel , zhong2020openmix, lee2020visualizing} &  \checkmark &  \checkmark &  \checkmark &   \xmark    &    \checkmark & \xmark   &  \xmark &   \xmark   &  \checkmark \\ 
\hline
Zero-shot learning  \cite{romera2015embarrassingly, xian2017zero, xian2018zero, bansal2018zero, guo2018zero, li2020symmetry, rohrbach2011evaluating, socher2013zero, elhoseiny2013write} &  \xmark & \checkmark & \xmark   &   \xmark  &   \xmark  & \xmark  & \xmark &   \xmark   & \checkmark  \\ 
\hline
Generalized zero-shot learning  \cite{mancini2021open, fu2019vocabulary, gune2019generalized, liu2018generalized, atzmon2019adaptive} &  \multirow{2}{*}{\xmark} &  \multirow{2}{*}{\checkmark} &  \multirow{2}{*}{\checkmark} &  \multirow{2}{*}{\xmark}    &   \multirow{2}{*}{\checkmark}  & \multirow{2}{*}{\xmark}  &  \multirow{2}{*}{\xmark}  &  \multirow{2}{*}{\xmark}  & \multirow{2}{*}{\checkmark} \\
\cite{ mandal2019out, bhattacharjee2019autoencoder, geng2020guided, chen2020boundary, zhang2018triple, rahman2018unified, huynh2020shared, verma2018generalized, kodirov2017semantic, felix2018multi, xie2020region}  & & & & & & & &  & \\
\hline
Class IL  \cite{shmelkov2017incremental, kemker2018fearnet, belouadah2018deesil, belouadah2019il2m, belouadah2020scail, ayub2020cbcl, ayub2020storing, ayub2021eec, guerriero18openreview}  &  \checkmark & \checkmark & \xmark  &  \xmark   &   \xmark  & \checkmark  & \xmark  &   \xmark  &  \checkmark  \\ 
\hline
Zero-shot Incremental Learning (IL)  \cite{xue2017incremental}  &  \xmark & \checkmark & \xmark  &   \xmark   &   \xmark & \checkmark  &  \xmark  &  \xmark  &   \xmark \\ 
\hline
Representation IL  \cite{li2019incremental, din2020online, liu2020semi}  &  \checkmark & \checkmark & \xmark  &  \xmark   &   \xmark  &  \xmark  & \checkmark  &   \xmark  &  \xmark  \\ 
\hline
Class and representation IL  \cite{wu2019large, rebuffi2017icarl, lopez2017gradient, castro2018end, hou2019learning} &  \multirow{2}{*}{\checkmark} & \multirow{2}{*}{\checkmark} & \multirow{2}{*}{\xmark}  &   \multirow{2}{*}{\xmark}   &   \multirow{2}{*}{\xmark} &  \multirow{2}{*}{\checkmark} &  \multirow{2}{*}{\checkmark}  & \multirow{2}{*}{\xmark}   & \multirow{2}{*}{\checkmark}  \\
\cite{zhao2020maintaining, he2020incremental, liu2020generative, liu2020mnemonics, douillard2020podnet, douillard2020plop, yu2020semantic, zhang2020class, mi2020generalized, kurmi2021not}  &   &  &   &  &   &   &   &   & \\
\hline
Class OWL  \cite{bendale2015towards, rudd2017extreme, boult2019learning, xu2019open, dhamija2021self, cao2021open}  &  \checkmark & \checkmark  & \checkmark   &    \checkmark  &   \xmark & \checkmark  &  \xmark  &   \xmark  &  \checkmark  \\
\hline
Zero-shot OWL   &  \xmark & \checkmark  & \checkmark   &    \checkmark  &   \xmark & \checkmark  &  \xmark  &  \xmark  &  \xmark  \\
\hline
Representation OWL&  \checkmark & \checkmark  & \checkmark   &  \checkmark   &   \xmark  &  \xmark &  \checkmark   &   \xmark  &  \xmark  \\
\hline
Class and representation OWL \cite{guo2019multi, joseph2021open} &  \checkmark & \checkmark  & \checkmark   &   \checkmark   &   \xmark & \checkmark  &  \checkmark  &  \xmark     &  \checkmark  \\
\hline
Open-set class IL &  \checkmark & \checkmark  & \checkmark   &   \xmark  &   \xmark  & \checkmark   &  \xmark   &  \xmark    &  \xmark  \\
\hline
Open-set zero-shot  IL  and Generalized zero-shot IL  \cite{jia2017incremental, feng2020transfer} &  \xmark & \checkmark  & \checkmark   &   \xmark  &   \checkmark  & \checkmark   &  \xmark   &  \xmark    &  \xmark  \\
\hline
Open-set representation IL &  \checkmark & \checkmark  & \checkmark   &  \xmark  &   \xmark  & \xmark   &  \checkmark    &  \xmark   &  \xmark   \\
\hline
Open-set class and representation IL &  \checkmark & \checkmark  & \checkmark   &   \xmark  &   \xmark  & \checkmark   &  \checkmark    &  \xmark   &  \xmark   \\
\hline
Unsupervised class IL  \cite{pernici2017unsupervised, pernici2020self,  lv2018unsupervised, kalshetti2019unsupervised, marxer2016unsupervised}  &  \checkmark &  \checkmark & \checkmark   &  \xmark    &   \xmark  & \checkmark  & \xmark  &  \checkmark   &  \xmark   \\ 
\hline
Unsupervised representation IL  \cite{rao2019continual, ma2019unsupervised, aljundi2019task}  &  \checkmark &  \checkmark &  \xmark  &  \checkmark   &   \xmark  &  \xmark  & \checkmark  &  \checkmark   &  \xmark   \\ 
\hline
Unsupervised class and representation IL   \cite{stojanov2019incremental, allred2016unsupervised}  &  \checkmark &  \checkmark &  \xmark  &  \xmark    &   \xmark &  \checkmark &  \checkmark  & \checkmark   &  \xmark  \\
\hline
Unsupervised open-set class IL & \checkmark  &  \checkmark & \checkmark  &  \xmark   &   \checkmark  & \checkmark  &  \xmark  &  \checkmark    &  \xmark  \\\hline
Unsupervised open-set zero-shot class IL   & \xmark  &  \checkmark & \checkmark  &  \xmark   &   \checkmark  & \checkmark  &  \xmark  &  \checkmark   &  \xmark  \\
\hline
Unsupervised open-set representation IL  &  \checkmark & \checkmark  & \checkmark  &   \xmark   &   \xmark  &  \xmark & \checkmark   &  \checkmark   &  \xmark  \\
\hline
Unsupervised open-set class and representation IL & \checkmark & \checkmark  & \checkmark  &   \xmark   &   \checkmark  &  \checkmark & \checkmark   &  \checkmark   &  \xmark  \\
\hline
{\bf\rule{0pt}{10pt} \rule[-6pt]{0pt}{6pt} Class OWL without labels (this paper)}  & \checkmark  &  \checkmark & \checkmark  &  \checkmark   &   \checkmark  & \checkmark  &  \xmark  &  \checkmark    &  \checkmark \\\hline
Zero-shot OWL without labels   (future works)  & \xmark  &  \checkmark & \checkmark  &  \checkmark   &   \checkmark  & \checkmark  &  \xmark  &  \checkmark  &  \xmark   \\
\hline
Representation OWL without labels (future works) &  \checkmark & \checkmark  & \checkmark  &   \checkmark   &   \xmark  &  \xmark & \checkmark   &  \checkmark  &  \xmark   \\
\hline
Class and representation OWL without labels (future works) & \checkmark & \checkmark  & \checkmark  &   \checkmark   &   \checkmark  &  \checkmark & \checkmark   &  \checkmark   &  \xmark  \\
\hline
\end{tabular}
\end{center}}
\label{table_compare_problems}
\end{table*}

\section{Survey of Related Work}
\label{sec:related}
In this section, we document some of the work in closely related research problems of out-of-distribution detection, open-set recognition, generalized zero-shot learning, supervised incremental learning, generalized novelty discovery, supervised open-world learning, and unsupervised incremental learning, providing examples of the formalization from the previous. We further discriminate open-world learning without labels with each of these related problems and provide a comparison for these related works in Table \ref{table_compare_problems} with the goal of providing broad coverage of recent literature. To the best of our knowledge, there is not any prior work on open-world learning without labels; we discuss the most relevant of the works from the table and how they are related to and different from our work.

\subsection{Novelty Detection and Open-Set Recognition}
Out-of-distribution detection \cite{liu2020energy}, anomaly detection \cite{ruff2019deep}, and novelty detection \cite{ zhang2020multi} are binary classifiers and cannot distinguish between known classes of the training data set. On the other hand, approaches developed for open-set recognition \cite{ oza2019c2ae} are algorithms that classify knowns and put all unknowns in "a reject class." Currently, Helmholtz free energy \cite{liu2020energy} is the state-of-the-art among binary classifiers and \cite{miller2021class} is the state-of-the-art in open-set recognition. Long-tail open-set recognition has also been studied recently in \cite{liu2019large}. Although the paper \cite{liu2019large} used the term open-world in the title, it never shows how to learn new classes. Out-of-distribution detection, anomaly detection, novelty detection, and open-set recognition algorithms do not distinguish between unknown classes. They can be viewed as a component of an open-world agent, but by themselves cannot accomplish open-world learning, definition \ref{def_OWL_CR}.

\subsection{Generalized Zero-Shot learning}

Zero-shot learning \cite{xian2018zero} is a set of algorithms that recognizes new concepts by just describing their unique set of attributes, i.e., each class is described by combinations of attributes and rules. In practice, the majority of agents that work in the open-world do not have any attribute extractor. By adding an open-set recognition to zero-shot learning, we can create generalized zero-shot learning. Thus, generalized zero-shot learning \cite{chen2020boundary} is an extension of zero-shot learning that classifies both knowns and unknowns. Both require attributes as their input and use rules to define classes, including some unseen classes. Agents in open-world learning usually do not have attribute extractors. So, comparing open-world learning with generalized zero-shot learning is similar to comparing apples and oranges. We did not add any generalized zero-shot learning to the result section because open-world learning does not have access to attributes.

In a recent paper \cite{mancini2021open},  efforts have been made to relate generalized zero-shot learning to open-world learning. Still, they have explicitly tried to redefine open-world learning, stating that their new definition is not consistent with \cite{bendale2015towards} and also inconsistent with our definition.  Their model requires that all classes be defined as a combination of primitives with all primitives provided during training. Then they define the open-world as a problem where all the combinations of states and objects can form a valid compositional class.  Presuming all primitives are available in the training stage is unrealistic, so the definitions are inconsistent. Thus we do not compare with any generalized ZSL approach.

\subsection{Supervised Incremental Learning}

Supervised incremental learning has been studied for many years \cite{rebuffi2017icarl}. Some researchers proposed to decouple feature extraction from inference subsystem \cite{belouadah2018deesil} while others proposed to adapt both subsystems jointly \cite{zhao2020maintaining}. Also, some researchers only add classes incrementally when feature extraction remains constant \cite{lesort2020continual}. Some models, such as artificial neural networks, forget what they have learned while learning new classes \cite{rolnick2019experience}. This phenomenon is known as \textit{catastrophic forgetting}. Fortunately, there are many models that never experience forgetting such as K Nearest Neighbor \cite{losing2018incremental}, naive Bayes \cite{losing2018incremental}, random forest\cite{losing2018incremental}, neural decision forests \cite{kontschieder2015deep, roy2020tree}, and support vector machine \cite{benavides2019svm}. Because agents in incremental learning receive labels, they do not detect which samples belong to unknown classes. Although incremental learning is a component of open-world, open-world learning cannot be compared directly with incremental learning, see definition \ref{def_OWL_CR}.

\subsection{Generalized Novelty Discovery}

Novelty discovery is the problem of exploring novel classes in an image collection given labeled examples of known classes  \cite{vo2019unsupervised, han2019learning, han2019autonovel,  vo2020toward, qing2021end}.  Novelty discovery is significantly harder than semi-supervised learning because there are no labeled examples for the new classes. Unsupervised clustering \cite{sarfraz2019efficient, mcinnes2017hdbscan} can be viewed as a special case of novelty discovery where agents ignore prior knowledge, i.e., known classes labeled in the training data set.  Novelty discovery agents assume that data are from unknown classes and do not belong to known classes.  However, open-world learning agents receive both known and unknown inputs and cannot assume a particular input is known or unknown. Thus, novelty discovery is not equal to open-world learning. By using open-set recognition before novelty discovery, we can create generalized novelty discovery. Generalized novelty discovery agents can classify knowns and cluster unknowns. Currently, \cite{han2019autonovel} is the state-of-the-art generalized novelty discovery. The main difference between generalized novelty discovery and open-world learning is that open-world learning requires a multi-round incremental algorithm learning new classes incrementally from a stream, see definition \ref{def_OWL_CR}. In contrast, generalized novelty discovery is not multi-round incremental. It uses only one batch.  Therefore, generalized novelty discovery is not equal to open-world learning. One way to create open-world learning is running generalized novelty discovery on each batch of data then doing incremental learning. This idea is impractical in large-scale applications because each batch takes a long time. Thus, we did not include it in the comparison.

\subsection{Supervised Open-World Learning}

Obviously, the most related work is supervised open-world learning, which was first formalized by \cite{bendale2015towards} with subsequent work in  \cite{rudd2017extreme, boult2019learning}.  {\em In these prior works, supervised labels were provided to support the incremental learning in the open-world, and labels were provided for ALL unknowns, not just those detected, providing unrealistically good performance.} The Extreme Value Machine (EVM) \cite{rudd2017extreme} is the state-of-the-art for ImageNet scale classification problems. In this paper, we extend the EVM to work without labels and amend the broken evaluation protocols. More detail in section \ref{section_method}.

Helmholtz free energy for detecting unknowns was proposed in \cite{joseph2021open}. Then they used contrastive clustering in latent space to distinguish between unknown classes. To mitigate catastrophic forgetting, they stored a balanced set of exemplars and fine-tuned their model. Although their results seem to be promising, contrastive clustering is not scalable.

Multi-stage deep classifier cascades \cite{guo2019multi} consisted of a root and several leaf classifiers trained in an end-to-end fashion. This method can increment the leaf nodes to both recognize a newly added class and detect future unknown classes. It can learn and include new features without disturbing the existing features. Their experiments were limited to RF signals. As we know, cascaded classifiers are not practical in large-scale applications.

Thresholding to a K-Nearest Neighbor classifier was proposed by \cite{xu2019open} as part of their open-world learning agent. They only experiment in a small-scale textual domain. The results in \cite{rudd2017extreme} show that EVM has better performance than the nearest neighbor in the image classification domain.

\subsection{Unsupervised Incremental Learning}

Unsupervised incremental learning has been used in many applications such as the prediction of musical audio signals \cite{marxer2016unsupervised}, hand shape and pose estimation \cite{kalshetti2019unsupervised}, Financial Fraud Detection \cite{ma2019unsupervised}, road traffic congestion detection \cite{bandaragoda2019trajectory}, etc. Unsupervised incremental learning algorithms can detect and learn unknown classes, but they cannot manage them, i.e., they learn them immediately even if they are a singleton or noise. Thus, according to definition \ref{def_OWL_CR}, they are not an open-world learner. In the following, we briefly describe the closest works. In \cite{lv2018unsupervised}, authors designed a person re-identification algorithm based on pedestrian spatial-temporal patterns in the target domain that consisted of a feature extractor (CNN) and a matching model (Bayesian). Temporal patterns are not accessible in many image classifier agents.
In \cite{allred2016unsupervised}, researchers proposed spike-timing-dependent plasticity for spiking neural networks to learn digits 0 to 9 incrementally. Both approaches had very limited experiments and may not work in more complex data such as ImageNet or Places 365 standard.

In excellent research \cite{pernici2017unsupervised}, VGGface has used feature extraction in videos. Followed by a modified version of the nearest neighbor to learn new faces (classes) incrementally. Also, they designed a feature forgetting strategy to control memory size in the long run. The results in \cite{rudd2017extreme} show that EVM has better performance than the nearest neighbor. Thus, in this paper, we do not use the nearest neighbor classifier.
In \cite{lopez2019incremental}, they compared Support Vector Machines (SVM) with Extreme Learning Machines (ELM) in the task of incremental learning for face verification in video surveillance. They found that ELM is slightly better than SVM.

Continual Recognition Inspired by Babies (CRIB) \cite{stojanov2019incremental} is an unsupervised incremental object learning environment that can produce data that models visual imagery produced by object exploration in early infancy. They reported that single exposure yields catastrophic forgetting. The algorithm’s accuracy stays constant or decreases with a greater number of objects, so it is not scalable. Their algorithm exploits 3D models, so it could not be used as a basis for comparison in this paper.

Unsupervised incremental learning looks very similar to open-world learning. However, it is different because it does not manage unknowns. In open-world learning, an agent tracks unknowns until it has enough samples to form a meaningful cluster. Then it only learns the good cluster and keeps the rest of the unknowns for the future. However, an unsupervised incremental learning agent creates and learns a singleton class as soon as it sees a novel object. Up until now, there is not any published work that has demonstrated large-scale unsupervised incremental learning. Thus, we exclude them in the comparison.

\subsection{Other "Open-World"  Papers}

Some papers used the term "open-world" in their titles, but they never showed how to learn new classes, so are inconsistent with the original definitions of \cite{bendale2015towards} as well as the definitions of this paper.  For example according to our formalization, \cite{liu2019large} is addressing open-set recognition;  \cite{lakkaraju2017identifying} is addressing novelty detection;  \cite{sehwag2019analyzing} is addressing out-of-distribution detection.  In \cite{mancini2021open} the term open-world has been used but explicitly tried to redefine it to mean their form of generalized zero-shot learning, but they never showed how to learn new classes.
Even our own lab's work \cite{jafarzadeh2021automatic} investigated the reliability of classifiers in an open-world setting,  but in that paper, it did not show incremental learning steps was more open-set than open-world.

While we can appreciate the colloquial usage in each of these papers and have fallen into the title trap ourselves,  we believe it is in the interest of the field to have consistent and precise definitions. We encourage authors to take the formalizations in this paper as a start and expand them for their variations.  For example, one could formalize long-tailed as a modifier world, which could then be applied as a modifier to most of the problems formalized in section~\ref{sec_formal}.

\section{A general method for open-world learning}
\label{section_method}

This section develops a general framework for building open-world learning algorithms with six main elements: deep feature extraction, known classifications, novelty detection, novelty discovery, novelty management, and modified incremental learning.    Most of the existing incremental learning work cannot be directly applied but based on the survey above. We select components and explain how to modify them.  The ideal systems include novelty management and are summarized in Algorithms \ref{Alg_TOWL_other} and \ref{Alg_TOWL_FEVM}.  We will also consider the versions without management, which immediately use the detected novel items.

\begin{algorithm}[!t]
\caption{True Open-world Learner - LC}
\label{Alg_TOWL_other}
\SetAlgoNoLine
\DontPrintSemicolon
\SetKwInput{KwData}{Config}
\SetKwInput{KwResult}{Initialize}
\KwIn{Single image, Residual set, and incremental learning model (ILM)}
\KwData{$\psi$=minimum \# images to start learning, $\gamma$ =minimum \# cluster to start learning, $\rho$=minimum cluster size to create a new class, quality assurance SVM, $\tau$ confidence threshold, out-of-distribution detector OOD, and Linear classier (LC), }
\KwResult{$\Omega$ incremental features as Null}
\KwOut{predicted class, new ILM, incremental features }
\;
x $\leftarrow$ normalize image\;
f $\leftarrow$ CNN($x$) \tcp*{Deep feature}
$ \ell\leftarrow$ LC(f) \tcp*{Logit}
q $\leftarrow$ SoftMax($\ell$) \tcp*{class probabilities}
s $\leftarrow 1 -$ OOD $(x, f, \ell, q)$ \tcp*{Confidence}

\uIf{s $ > \tau$}{
p $\leftarrow (0, q, 0, 0 , \dots, 0) $
}
\Else{ 
u, d $\leftarrow$ ILM(f)\;
\tcp{ u is unknown probability}
\tcp{ d is discovered probability}
p $\leftarrow (d, 0, 0 , \dots, 0, d) $ \;
\tcp{ p is normalize probabilities}
\uIf{$u > \max(d)$}{
 \textbf{Insert} f in Residual 
}
}

y $\leftarrow$ argmax (p) \tcp*{Predicted label}

 T $\leftarrow \{ \}$\;
\uIf{size(Residual) $> \psi$}{ 
 L $\leftarrow$ Clustering(Residual)\;
 \tcp{L: cluster labels}
 \uIf{M $> \gamma$}{ 
 \ForEach{cluster K}{
 \uIf{size(K) $> \rho$}{
 F $\leftarrow \{a \in \textup{Residual} \; | \; K \}$\;
 c $\leftarrow \textup{mean}(F)$\;
 $v_e \leftarrow$ variance of F from c\;
 $v_c \leftarrow$ cosine variance of F from c\;
  \uIf{ SVM $ \; (v_e, v_c) > 0$}{ 
  Insert F to T
  }}}}}
\textbf{Update} IL with T\;
Delete covered clusters from Residual\;
$\Omega \; \leftarrow \; \Omega \; \cup $  T\;
\Return y, ILM, Residual, and $\Omega$
\end{algorithm}

\begin{algorithm}[!t]
\caption{True Open-world Learner - FEVM}
\label{Alg_TOWL_FEVM}
\SetAlgoNoLine
DontPrintSemicolon
\SetKwInput{KwData}{Config}
\SetKwInput{KwResult}{Initialize}
\KwIn{Single image, Residual set, and EVM model}
\KwData{$\psi$=minimum \# images to start learning, $\gamma$ =minimum \# cluster to start learning, $\rho$=minimum cluster size to create a new class, $\Omega$=pre-trained features, quality assurance SVM }
\KwOut{predicted class, new EVM, post-trained features }
\;
x $\leftarrow$ normalize image\;
f $\leftarrow$ CNN($x$) \tcp*{Deep feature}
q $\leftarrow$ EVM(f) \tcp*{class probabilities}
v $\leftarrow$ concatenate ( 1 - $\max$(q) , q ) \;
p $\leftarrow  v /\sum v $ \tcp*{Normalize probabilities}
y $\leftarrow$ argmax (p) \tcp*{Predicted label}
$u \leftarrow$ first element of p \tcp*{Unknown proba.}
\uIf{$u > \max(q)$}{
 Insert f in Residual 
}
 T $\leftarrow \{ \}$\;
\uIf{size(Residual) $> \psi$}{ 
 L $\leftarrow$ Clustering(Residual)\;
 \tcp{L: cluster labels}
 \uIf{M $> \gamma$}{ 
 \ForEach{cluster K}{
 \uIf{size(K) $> \rho$}{
 F $\leftarrow \{a \in \textup{Residual} \; | \; K \}$\;
 c $\leftarrow \textup{mean}(F)$\;
 $v_e \leftarrow$ variance of F from c\;
 $v_c \leftarrow$ cosine variance of F from c\;
  \uIf{ SVM $ \; (v_e, v_c) > 0$}{ 
  Insert F to T
  }}}}}
\textbf{Update} EVM with T\;
Delete covered clusters from Residual\;
$\Omega \; \leftarrow \; \Omega \; \cup $  T\;
\Return y, EVM, Residual, and $\Omega$
\end{algorithm}

\subsection{Feature Extraction and Classifiers for Knowns}
\label{subsection_feature_extraction}

The first subsystem of open-world learning is feature representation. Here, we used deep convolutional neural networks (CNN).  CNN can be easily replaced by other types of feature extractors such as transformers.
A CNN was trained on a training data set using supervised cross-entropy loss. While another CNN was trained on an auxiliary data set using self-supervised learning. Then  we extracted features from the latent layer, i.e., the last layer of each network before Logit. The proposed feature extractor is the concatenation of the frozen features of the two networks.

One of the proposed baselines uses a single classifier that classifies both $K_0$ knowns and $K_k - K_0$ new discovered classes. We trained a single EVM for this baseline. The other six baselines have two independent classifiers: a $K_0$ knowns classifier and a $K_k - K_0$ discovered classifier. For the $K_0$ known classifier, we trained a linear classifier with SoftMax activation function and cross-entropy loss in a supervised manner. The discovered classifiers are described in subsection \ref{subsection_incremental_learning}.

\subsection{Novelty Detection}
\label{subsection_novelty_detection}

To detect novelty, in the six baselines that have an independent linear classifier for $K_0$ knowns, we use thresholding over the SoftMax score. Similar to \cite{hendrycks17baseline}, the threshold is selected to have 95\% true positive rate on the knowns in the validation set.

For the baseline that has only one classifier, we use thresholding over the maximum EVM class probability, TPR = 0.95. EVM is derived from Extreme Value Theory (EVT), which is a branch of statistics that studies the behavior of extreme events and their associated probability distributions \cite{coles2001introduction, beirlant2006statistics, castillo2012extreme}. EVT is an extrapolation from observed samples to unobserved samples providing estimates of the probability of events that are more extreme than any of the already observed ones.  There are two principal parametric approaches to modeling the extremes of a probability distribution: (1) block maxima and (2) threshold exceedance. The Hill Estimator approach is also commonly used which is a non-parametric approach. The block maxima uses Generalized Extreme Value (GEV) distribution and threshold exceedance uses Generalized Pareto Distribution (GPD). According to Fisher-Tippet asymptotic theorem, for normalized maxima of blocks of random variables $M_n= \max(X_1, ..., X_n)$, there is a non-degenerate distribution, which is a GEV distribution, which for our case must follow  a Weibull distribution
\begin{equation}
\textup{W} (x; \mu , \sigma , \xi) = 
\begin{cases}
e^{- (1 + \xi(\frac{x - \mu}{\sigma}))^{\xi}} & , x< \mu - \frac{\sigma}{\xi} \\ 
1 & , x \geq \mu - \frac{\sigma}{\xi} 
\end{cases}
\end{equation}
The Extreme Value Machine (EVM) \cite{rudd2017extreme, henrydoss2017incremental} is a distance-based kernel-free non-linear classifier that uses Weibull families distribution to compute the radial probability of inclusion of a point with respect to nearest members of other classes. For a given point $x_i$, it fits the Weibull on the distribution margin distance, half the distance to the nearest negative samples
\begin{equation}
m_{i,j} = 0.5 * \|\hat{x}_i - x_j\|
\label{eq_margin}
\end{equation}
for the $\tailsize$ closest points $x_j$ from other classes. 
EVM provides a compact probabilistic representation of each class’s decision boundary, characterized in terms of its extreme vectors. Each extreme vector has a family of Weibull distribution. The probability of a point belonging to each class is defined as the maximum probability of the point belonging to each extreme vector of the class. EVM uses greedy approximation for Karp’s set cover problem for model size reduction by deleting redundant extreme vectors. In short, EVM for each input (point) computes the probability of inclusion to each class, i.e., the output is a vector of probabilities. The predicted class is computed by
\begin{align}
\label{EVM}
\hat{P} (C_l|x) &= \max_{j} \textup{W}_{l,j} (x; \mu_{l,j} , \sigma_{l,j} , \xi_{l,j}) 
\end{align}
where $\textup{W}_{l,j} (x)$ is Weibull probability of $x$ corresponding to $j$ extreme vector in class $l$.

\subsection{Novelty Discovery}
\label{subsection_novelty_dicscovery}

All prior published researchers pursued supervised open-world learning, getting labels for all novel images, and updating the model. The fundamental flaw is that the updated model does not depend on the past performance of the agent on detecting novel classes. However, to develop our True Open-World Learner (TOWL), we proposed to collect nominated novel images and group them to create new classes. While one might consider classic clustering, such as K-Means, we do not have any prior expectations on the number of new classes. Automatically discovering related groups of data in unsupervised data without parameters is an important and still unsolved problem. There are only a few published novelty discovery methods that are appropriate. In this paper, we use the Finch algorithm \cite{sarfraz2019efficient} for novelty discovery, which, while it is formally parameter-free, still requires the user to select among different returned partitions.  Finch can be replaced with other clustering or novelty discovery algorithms as part of novelty discoveries. For example, Finch can be replaced by AutoNovel \cite{han2019autonovel}, which is state-of-the-art but very slow.

\subsection{Novelty Management}
\label{subsection_novelty_management}

\begin{figure}[!t]
\centering
\includegraphics[width=\linewidth]{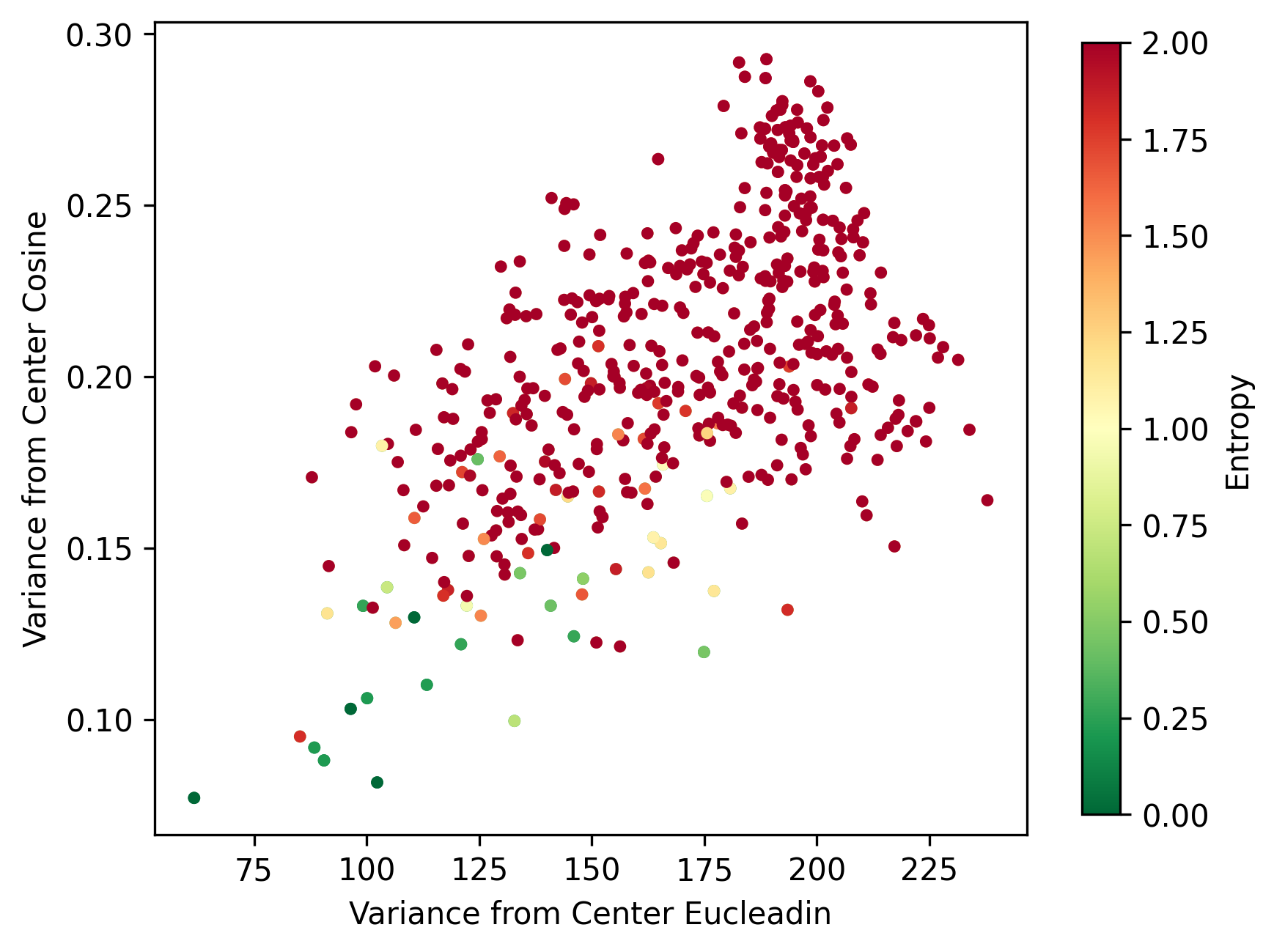}
\caption{Entropy of clusters of validation test plotted against Euclidean and cosine variance. As can be seen, the clusters with low ground-truth labeled entropy, i.e. more consistent with ground-truth labeling, are moderately well separated in the plot, suggesting a classifer can detect them. }
\label{fig_svm} 
\end{figure}

Novelty management can be divided into three phases: pre-discovery novelty management, intra-discovery management, and post-discovery novelty management. The pre-discovery novelty manager is a simple subsystem that is responsible to make sure the input of the novelty discovery process is sufficient. We proposed to collect nominated novel images into a residual set. When the set size becomes greater than a threshold, we pass the entire set to novelty discovery.

The intra-discovery manager is responsible to select proper hyper-parameters for the novelty discovery algorithm. Although Finch is parameter-free, it still does not provide a fully automatic operation since it produces multiple potential partitions among which we must choose. The order of partition is reversely sorted by the number of clusters, i.e., the first partition has the largest number of clusters and the last partition has the smallest number of clusters. The proposed manager selects the first partition if Finch outputs a/two partition(s); otherwise, it selects the second partition of the Finch.

The most challenging part of open-world learning is post-discovery novelty management. There might be several clusters that the number of data points in that cluster is smaller than the minimum requirement for the incremental learner, such as singletons and pairs for many algorithms. For this issue, the manager can use a threshold, for example, five for GMM. Moreover, there might be several clusters with the same class, i.e., over-clustering. Similarly, the manager can use a threshold, for example, 20. Thus, for solving these two issues, TOWL uses the maximum of these 2 numbers, i.e., 20 because it is the maximum of 5 and 20. Besides, there might be several clusters where each of them consists of several categories, i.e., under-clustering. This issue is more prevalent in novelty discovery that uses nearest neighbors. For example, if the input is a mixture of images of birds, flowers, fish, and airplanes, there is a high probability that birds and flowers create a cluster because there are some images that have both flowers and birds. We observe that there is a correlation between the entropy of clusters and Euclidean and cosine variance of cluster points (Fig. \ref{fig_svm}). Therefore, to reject clusters with multiple categories, we proposed that computing Euclidean and cosine variance of cluster points and using them as input of a Support Vector Machine (SVM) with a linear kernel to reject the clusters with a large value of entropy, which in our experiments is the minimum entropy cluster in validation run.

In the proposed novelty management, any images that do not form a cluster or belong to a rejected cluster remains in memory for the future until it passes to the incremental step, i.e., will be covered by a class of EVM.

\subsection{Incremental Learning}
\label{subsection_incremental_learning}

At each step, the novelty manager may pass one or a few clusters to be learned incrementally. In the following subsection, we describe how we modify well-known incremental learning algorithms to work in open-world settings. Similar to incremental learning, open-world learning also suffer from catastrophic forgetting. Catastrophic forgetting means the classifier forgets previous classes during learning new classes over time. Freezing the feature extractor alleviates catastrophic forgetting but cannot solve it completely. Thus, we study algorithms that are using frozen feature extractors. The study of open-world representation learning is beyond the scope of this paper.

\subsection{Modification of Past Works}
\label{subsection_modification}

\subsubsection{ONCM:Open-world NCM}
\label{subsection_modifed_ncm}

The nearest classifier means \cite{guerriero2018ncm} is an incremental learning algorithm that stores the mean of features of each class. A new class can be easily added by adding the mean of its features to the memory. In the inference, the Euclidean distances of the feature of the test sample from each means in the memory are computed. The predicted class is the class that has the minimum Euclidean distance. Although it is a very powerful algorithm in incremental learning problems, it cannot be used directly in open-world learning problems because the volume of known space is infinite, and the predicted class is one of the existing classes, regardless of the magnitude of the distance. Therefore, the modification of NCM is necessary for the open-world problem.

The first modification of NCM is to output unknowns if the number of means in memory is less than three. This, modification is necessary to avoid the singularity. The second modification is the computation of probability in the inference time 
\begin{equation}
\label{EQ_ncm_1}
q_i(x) = \frac{e^{-\| x - \mu_i \|_2}}{\sum\limits_j e^{-\| x - \mu_j \|_2}}
\end{equation}
We compute the SoftMax of negative distance. Then we define confidence as the maximum value of SoftMa, so uncertainty is one minus the maximum SoftMax value. The augmented possibility is the concatenation of uncertainty and SoftMax. The output of our ONCM is normalized augmented possibility
\begin{equation}
\label{EQ_ncm_2}
p(x) = \frac{[ \; 1 \; - \;  \max (q) \; , \; q^T \; ]^T}{1 \; -  \;  \max (q) \;  + \sum\limits_j q_j}
\end{equation}
where $\max(.)$ is an operator that act on a vector and output the maximum magnitude of a element among all elements of the vector.

\subsubsection{ONNO: Open-world NNO}
\label{subsection_modifed_nno}

Nearest Non Outlier (NNO) \cite{bendale2015towards} is a combination of the metric learning algorithm and out-of-distribution detection. NNO uses Nearest Class Mean Metric Learning (NCMML) \cite{mensink2013distance}, which is a metric learning for a classification task when inputs come from a frozen high dimensional feature extractor. NCMML uses a learnable matrix, called metric, $W_k \in \mathbb{R}^{f \times r}$ where $r$ is a significantly smaller than  feature dimension $r << f$.  The probability can be computed by 
\begin{equation}
\label{EQ_nno_1}
q^{NCMML}_i(x) = \frac{ e^{- \frac{1}{2} \; ( x - \mu_i)^T \; W_k^T \; W_k \; ( x - \mu_i) } }{ {\sum\limits_j e^{- \frac{1}{2} \; ( x - \mu_j)^T \; W_k^T \; W_k \; ( x - \mu_j) } }  }
\end{equation}
where $\mu_i$ is mean of all feature  points in the $i$'th class. At each time steps $k$, the metric $W_k$ is trained by minimizing the cross-entropy loss of prediction (\ref{EQ_nno_1}).  Then NNO computes 
\begin{align}
\label{EQ_nno_2}
\begin{multlined}
q_i(x) =  q^{NCMML}_i(x) \;\; \textup{Heaviside} \; (1 -  \\ \frac{1}{\tau} \; ( x - \mu_i)^T \; W_k^T \; W_k \; ( x - \mu_i))
\end{multlined}
\end{align}
where $\tau \in \mathbb{R}^+$ is constant, is called scale factor.  The value of $\tau$ is optimized in validation run, which is equal to 2 when $r = 4$ in our experiments. Finally, it estimates the probability of classes with
\begin{equation}
\label{EQ_nno_3}
p(x) = 
\begin{cases}
[1, 0 , 0 , \dots , 0 , 0]^T   & , \; \max(q) > 0 \\
[0, \; q^T  \; ]^T & , \;  \max(q) = 0
\end{cases}
\end{equation}
NCMML and NNO are powerful metric learning algorithms, but they have a singularity issue in open-world problems when the number of classes is less than three. Similar to the NCM, our modification to produce ONNO is to output unknowns. If the number of means in memory is less than three, the output will be an unknown (undiscovered) class. 

\subsubsection{OGMM:Open-world GMM}
\label{subsection_modifed_gmm}
Gaussian mixture model \cite{arandjelovic2005gmm} is an incremental learning algorithm that stores the mean and the covariance of each class. Similar to NCM, it is powerful in incremental learning problems but it has a singularity issue in open-world problems when the number of classes is less than three. In addition to NCM issues, GMM suffers from singularities of covariances. Thus, the modification of GMM is necessary for the open-world problem.

Similar to the NCM, the first modification is to output unknowns, and if the number of means in memory is less than three, the output will be an unknown (undiscovered) class. The second modification is computing the inverse of the covariance and prediction probability. If the determinant of a covariance $\Sigma$ is greater than a negligible number, $\epsilon = 0.0001$, we use the regular matrix inversion method to compute the inverse of the covariance. Otherwise, first, we compute the Moore–Penrose inverse of the covariance. Then we find the maximum value of the last steps. Finally, the inverse of covariance is equal to 1000 times the maximum value for zero diagonal and equal to the Moore–Penrose inverse for other elements
\begin{equation}
\label{EQ_gmm_1}
S^{-1} = 
\begin{cases}
\Sigma^{-1}      & , |\Sigma| \geq \epsilon \\
\Sigma^{\dagger}  + 1000 \max(\Sigma^{\dagger}) \;  D(\Sigma^{\dagger} \stackrel{?}{=} 0)  & , |\Sigma| < \epsilon
\end{cases}
\end{equation}
where  $\max(.)$ is the maximum operator on all elements of a matrix, $\dagger$ is Moore–Penrose operator.

There is not any algebraic closed-form solution to compute the cumulative density probability of  Gaussian with positive semi-definite covariance. Therefore, for OGMM we propose to use the following equation along with (\ref{EQ_ncm_2})

\begin{equation}
\label{EQ_gmm_2}
q_i(x) = \frac{e^{-\frac{(x - \mu_i)^T S^{-1} (x - \mu_i)}{2 s}}}{\sum\limits_j e^{-\frac{(x - \mu_j)^T S^{-1} (x - \mu_j)}{2 s}}}
\end{equation}
where $s \in \mathbb{R}^{+}$ is scale factor. The value of $s$ depends on the feature extractor and is tuned on the validation set. In this paper, (EfficientNet-B3 , ImageNet), $s$ is equal to ten.

\subsubsection{OCBCL: Open-world CBCL}
\label{subsection_modifed_cbcl}

Centroid-based concept learning \cite{ayub2020cbcl} is the state-of-the-art of distance-based incremental learning. For each class, it starts from a single point (feature), where the average is equal to the prototype itself. For each remaining feature in the class, it computes the distance of the feature to all current centers in the class, if the minimum distance is greater than a pre-defined threshold, the new point is added as a new center; otherwise, the closest center is updated. At inference time, it computes the distance of the test feature to all centers of all classes. Then the top k closest is selected. The score of each class is zero if it is not in the top k closest points, otherwise, it is the sum of the inverse of distance weighted by the number of samples belonging to the class. The k in top k is tuned on validation run, which in our experiment is equal to five.

CBCL has both issues discussed above for NCM. Therefore, we modified CBCL to solve the open-world problem. Similar to the NCM, the first modification is to output unknowns if the number of means in memory is less than three. The second modification is prediction probability. The score of OCBCL is modified to 
\begin{equation}
\label{EQ_cbcl_1}
r_i(x) = 
\begin{cases}
- \infty     & , \forall j \quad i_j \notin top \; K \\
\sum\limits_{i_j \; \in \; top \; K} \frac{1}{n_{i} \| x - \mu_{i_j} \|_2} & , \exists j \quad i_j \in top \; K 
\end{cases}.
\end{equation}
Then the normalized score can be calculated by
\begin{equation}
\label{EQ_cbcl_2}
q_i(x) = \frac{e^{r_i}}{\sum\limits_j e^{r_j}} .
\end{equation}
Finally, the prediction probability can be found using (\ref{EQ_ncm_2}).

\subsubsection{OSCAIL: Open-world SCAIL}
\label{subsection_modifed_scal}

In contrast with the above algorithms that use distance for incremental learning, scaling incremental learning \cite{belouadah2020scail} uses a linear classifier. It has a playing buffer that stores a limited number of features from past seen classes. In each step, SCAIL learns a new linear classifier and stores the initial mean absolute bias and mean sorted absolute weight of the newly added classes. For weight 
\begin{equation}
\label{EQ_scail_1}
\mu^{new} = \frac{1}{N^{new}} \; \sum\limits_{i} \underset{j}{\operatorname{sorted}} ( \{|w_{i_j}| \} ) )
\end{equation}
where $w_{i_j}$ is j'th weight of i'th classifier, and $N^{new}$ is number of new classes. For bias, j is constant 1, i.e., $\mu^{new} = \frac{1}{N^{new}} \sum_i |b_i|$. Then it scales the old classes' biases and weights by ratio of mean sorted absolute value new classes over initial value stored in memory 
\begin{equation}
\label{EQ_scail_2}
w_{i_j}^{scaled} = \frac{\mu_{r(j)}^{new}}{\mu_{r(j)}^{i}} w_{i_j}^{current}
\end{equation}
where $r(j)$ is the rank of j'th element of the feature vector.
Finally, it updates the play buffer. SCAIL has both issues of NCM.  Therefore, we modified SCAIL to solve the open-world problem. The first modification is to output unknowns if the number of means in memory is less than three. The second modification is prediction probability. For OSCAIL, we use (\ref{EQ_ncm_2}) where $q$ is the output of the original SCAIL.

\subsubsection{MEVM: Modified EVM}
\label{subsection_modifed_evm}

As mentioned in the earlier subsections we modified EVM. The Weibull family distribution often converges to zero rather quickly, and EVM generates a very sharp boundary. Thus, we declare an image as novel (and hence nominate it to create a new class) if the probability of the class of unknowns of MEVM (the first class, i.e., the class with label zero) is more than other classes. 
The original EVM formulation with its margin theorem concept, using Eq.~\ref{eq_margin}, is somewhat problematic for true open-worlds. The intuition behind the margin is that EVM is claiming half the space to the nearest other known class. That is fine for well-separated known classes, but it can easily be taking over too much open space for open-world learning as the assumption implies there are no classes in between the class being fitted and the nearest known classes. Because the original EVM experiments were tested using subsets of ImageNet using pre-trained features that already separated all classes, this oversight may not have been apparent. Also, we find that margin with a fixed tail size is poorly defined in highly imbalanced settings where a new class may have only a few samples. In such settings, we may need greater generalization from the few samples. Again this was not a problem in experiments in \cite{rudd2017extreme} which used nearly balanced sampling; real open-world learning cannot presume well-separated classes or balanced data. To address these issues, our enhanced MEVM includes the idea of a distance multiplier $d_m$, which replaces the multiplier of $0.5$ in Eq.~\ref{eq_margin} with a free parameter. If $d_m < 0.5$, then the model is smaller (more specialized), leaving some room between it and the nearest other known class. If we choose a higher value for distance multiplier $d_m > 0.5$ during incremental class addition, we can expand the class generalizing.  We tested on a range of values of the distance multiplier, optimizing open-set classification accuracy using held-out validation data. Among them, 0.45 demonstrates the best separation between known validation and unknown validation sets of ImageNet, slightly less generalization than the original EVM paper. With $d_m=0.45$, we leave some room for classes between two known classes.

We create two baselines from MEVM: a baseline with a linear classifier for known classes and EVM for discovered classes (LCMEVM), and another baseline with a single EVM for both knowns and discovered classes (FMEVM). The first modification in EVM is to add (\ref{EQ_ncm_2}) to EVM, where, $q$ is the class probability of the original EVM. The second modification is increasing the distance multiplier for discovered classes because, in most practical applications, the number of data in each class in the pre-training phase is significantly larger than the number of data in each cluster in the incremental phase, in order of 10 to 100. Thus, if we used the same parameters in the incremental phase as in the pre-training phase, the new clusters cannot generalize well. To solve this issue, we proposed to use a higher distance multiplier for new clusters than the pre-training classes. For the baseline with the linear classifier (LCMEVM), the pre-trained features set is initiated as a null set. In the other baseline (FMEVM), we proposed to load pre-trained features or post-trained features from the past and use it as negative examples for training the current class. In both baselines, if the novelty manager passes two or more clusters at a single step, in addition to the pre/post-trained feature, we use other clusters' features as negative examples for learning the current class.

A naive but simplest and fastest solution is to create a Weibull distribution with pre-defined parameters for either center or each point of each cluster independently and add it to EVM. Depending on the pre-defined parameter, it may overgeneralize or not generalize enough because the boundary of each class depends on categories and is not constant. Also, it causes the most catastrophic forgetting. Therefore, we did not use this simple method.

 One way of reducing catastrophic forgetting is combining a frozen feature extractor and using a MEVM distance multiplier of less than 0.5 in both pre-training and incremental phases. However, a distance multiplier of less than 0.5 results in not generalizing well. Another way to avoid catastrophic forgetting is to throw away existing MEVM and train all classes of MEVM in each step. Although training from scratch is the best solution on paper,  it takes a very long time and thus is not practical for many applications. Therefore, in this paper, in considering the trade-off between catastrophic forgetting, speed, and, generalization, our MEVM accepts some amount of forgetting but strives to keep it from being catastrophic.

\section{Evaluation Metrics}
Because open-world learning mixes recognition of known and unknown classes, directly applying traditional metrics designed for either supervised or unsupervised learning does not necessarily work well. 
Accuracy and balanced accuracy are the most popular metrics in supervised learning research. Unfortunately, accuracy cannot be defined when we do not have labels and cannot be applied to the unknowns. Even if we have ground-truth labels for the data that goes into the unknowns used in testing since no label is provided, the unsupervised learning may split classes or merge them. Hence, we need to include at least some unsupervised metrics, a.k.a clustering metrics.

B3 and Normalized Mutual Information (NMI) are the two most widely used metrics in clustering research \cite{amigo2009comparison}. B3 and NMI are good metrics when the number of samples is large enough to represent each class's probability distribution. In our experiments, we found that B3 and NMI on batches of data were not well suited to open-world learning, where we may have a large number of classes but only a small number of samples. None of them captures misclassifications of the unknowns into an otherwise empty "known" class or the splitting of a known class into a mix of known plus unknown classes, e.g., breaking novel views into new classes.   Therefore, we are proposing a new metric to overcome the issue of accuracy, B3, and NMI in open-world learning without labels. We call this the Open-World Metric (OWM). Before we define OWM we briefly summarize the B3 metric. 

\subsection{B3 Metric}
\label{subsection_metric_b3}
B3 is a fuzzy probabilistic metric that measures the precision and recall between clustering labels and true labels. Let us denote the features matrix with $X$ such that each row is a feature vector that corresponds to a point (sample). Then we can show the membership function of the true label with $\mu_Y(X)$ where the element in row $i$ and column $j$ is membership of point $i$ belonging to true class label $j$. Similarly, the clustering label's membership function can be shown by $\mu_K(X)$. Let us represent element-wise multiplication by $\odot$, element-wise multiplication division by $\oslash$, and a vector with all elements equal to one by $\mathbb{1}$. We can compute B3 metrics by:

{\small
\begin{align}
A_{L \times C} = \mu_Y^\top \; \mu_K &\qquad M_{L \times C} = A \odot A \\
T_{C \times 1} = \sum\limits_L A &\qquad S_{L \times 1} = \sum\limits_C A \\
P_{C \times 1} = (\sum\limits_L M) &  \oslash (T \odot T) \\
R_{L \times 1} =  (\sum\limits_C M) & \oslash (S \odot S) \\
\textup{Precision} = \frac{T^\top \, P}{T^\top \, \mathbb{1}} &\qquad \textup{Recall} = \frac{S^\top \, R}{S^\top \, \mathbb{1}}\\
F = \frac{2 \,\, \textup{Precision} \, . \, \textup{Recall}}{\textup{Precision} + \textup{Recall}} \label{eq_b3}
\end{align}
}

\subsection{Metric for Open-World Learning}
\label{subsection_metric_owl}

\begin{definition}{Open-world metric}
\\ Let $N$ be the number of items to be evaluated in data $X$. Let Acc be accuracy for known data and \textup{B3} be the B3 metric (Eq. \ref{eq_b3}) for unknown data. Let us use subscripts ground-truth and predicted categories of known and unknowns such that known predicted as known is $_{KK}$, known data which was (incorrectly) predicted as unknown by the classifier with $_{KU}$, unknown data that (incorrectly) predicted as known as $_{UK}$, and unknown data that predicted unknown by the classifier with $_{UU}$. For correct known predictions, we can use accuracy and for correct unknown predictions, we can use \textup{B3}, and we use incorrect predictions only in normalizing, then  the OWM score is computed by
\begin{equation}
\label{EQ_metric}
\textup{OWM} = \frac{ N_{KK} \;\; \textup{Acc}(X_{KK}) \;\; + \;\; N_{UU} \;\; \textup{B3}(X_{UU})}{N_{KK} \; + N_{KU} \; + N_{UK} \; + N_{UU}}
\end{equation}
\end{definition}
While we prefer B3, this measure can be generalized to combine other supervised or unsupervised metrics, e.g., $\textrm{OWM}_{\textrm{F1, NMI}}$ would use the above definition with macro-F1 instead of accuracy and NMI instead of B3.  We modified the library from paper \cite{baldwin1998description} to compute B3 scores.

\begin{table*}[!t]
\caption{Mean and standard deviation  on 5 tests, open-world scores of last 1000 images (10 batches). Feature is concatenation of frozen feature extractors were trained on ImageNet 2012  (supervised) and Places 365-standard (using MoCo v2).  Bold shows a "best" result that is statistically significantly better than others; italics shows a best results but  not statistically significantly better. }
{\small
\begin{center}
\begin{tabular}{|c|c|c|c|c|c|c|c|c|} 
\hline 
\# Unknown classes & \multicolumn{2}{c|}{5} & \multicolumn{2}{c|}{10} & \multicolumn{2}{c|}{25} & \multicolumn{2}{c|}{50} \\
\hline 
Method & $\mu$ & $\sigma$  &  $\mu$ & $\sigma$  &  $\mu$ & $\sigma$  &  $\mu$ & $\sigma$ \\
\hline \hline
TOWL-ONCM (ours) &  0.5629 & 0.0190 & 0.4960 & 0.0295 & 0.4505 & 0.0184 & 0.4374 & 0.0096 \\ 
\hline
TOWL-ONNO (ours) &  0.5248 & 0.0065 & 0.4782 & 0.0204 & 0.4430 & 0.0090 & 0.4335 & 0.0077 \\ 
\hline
TOWL-OGMM (ours) &  0.4963 & 0.0075 & 0.4454 & 0.0137 & 0.4237 & 0.0064 & 0.4207 & 0.0063 \\ 
\hline
TOWL-OCBCL (ours) &  0.5609 & 0.0247 & 0.4955 & 0.0401 & 0.4466 & 0.0128 & 0.4324 & 0.0077 \\ 
\hline
TOWL-OSCAIL (ours) &  0.5444 & 0.0304  &   0.4728 & 0.0169  &  0.4376 & 0.0074 & 0.4286 & 0.0069 \\ 
\hline
 TOWL-LCMEVM (ours) &   0.5541 & 0.0143 & 0.5012 & 0.0290 & 0.4606 & 0.0074 & \textit{0.4419} & 0.0125 \\ 
\hline
TOWL-FMEVM (ours) &  \textbf{0.5936} & 0.0174 & \textbf{0.5568} & 0.0343 & \textit{0.4670} & 0.0087 & 0.4040 & 0.0084 \\ 
\hline
\end{tabular}
\end{center}}
\label{table_compare_baselines_1}
\end{table*}

\begin{table*}[!t]
\caption{Mean and standard deviation  on 5 tests, open-world scores of last 1000 images (10 batches). Feature is concatenation of frozen feature extractors were trained on ImageNet 2012  (supervised) and Places 365-standard (using MoCo v2) }
{\small
\begin{center}
\begin{tabular}{|c|c|c|c|c|c|c|c|c|} 
\hline 
\# Unknown classes & \multicolumn{2}{c|}{5} & \multicolumn{2}{c|}{10} & \multicolumn{2}{c|}{25} & \multicolumn{2}{c|}{50} \\
\hline 
Method & $\mu$ & $\sigma$  &  $\mu$ & $\sigma$  &  $\mu$ & $\sigma$  &  $\mu$ & $\sigma$ \\
\hline \hline
SM OOD  + LC + ONCM + Finch FP  &  0.4922 & 0.0086 & 0.4477 & 0.0127 & 0.4215 & 0.0038 & 0.4178 & 0.0061  \\
\hline
Energy OOD  + LC + ONCM + Finch FP  &  0.4586 & 0.0134 & 0.4272 & 0.0103 & 0.4103 & 0.0059 & 0.4127 & 0.0050  \\
\hline
EVM OOD  + LC + ONCM + Finch FP  &  0.4880 & 0.0110 & 0.4389 & 0.0091 & 0.4139 & 0.0068 & 0.4108 & 0.0039  \\
\hline
SM OOD  + LC + ONNO + Finch FP  &  0.5097 & 0.0086 & 0.4673 & 0.0192 & 0.4593 & 0.0062 & 0.4560 & 0.0110  \\
\hline
Energy OOD  + LC + ONNO + Finch FP  &  0.4618 & 0.0131 & 0.4378 & 0.0118 & 0.4253 & 0.0087 & 0.4248 & 0.0054 \\
\hline
EVM OOD  + LC + ONNO + Finch FP  & 0.4992 & 0.0084 & 0.4609 & 0.0134 & 0.4392 & 0.0088 & 0.4407 & 0.0062  \\
\hline
SM OOD  + LC + OGMM + Finch FP  &  0.4891 & 0.0075 & 0.4418 & 0.0124 & 0.4196 & 0.0033 & 0.4174 & 0.0061  \\
\hline
Energy OOD  + LC + OGMM + Finch FP  &  0.4577 & 0.0129 & 0.4268 & 0.0100 & 0.4093 & 0.0053 & 0.4127 & 0.0050  \\
\hline
EVM OOD  + LC + OGMM + Finch FP  &  0.4861 & 0.0097 & 0.4356 & 0.0107 & 0.4110 & 0.0050 & 0.4101 & 0.0042  \\
\hline
SM OOD  + LC + OCBCL  + Finch FP  &  0.5518 & 0.0215 & 0.4801 & 0.0386 & 0.4393 & 0.0056 & 0.4286 & 0.0068  \\
\hline
Energy OOD  + LC + OCBCL  + Finch FP  &  0.4664 & 0.0159 & 0.4396 & 0.0145 & 0.4176 & 0.0046 & 0.4189 & 0.0063  \\
\hline
EVM OOD  + LC + OCBCL  + Finch FP  &  0.5254 & 0.0128 & 0.4769 & 0.0246 & 0.4208 & 0.0072 & 0.4212 & 0.0061  \\
\hline
SM OOD  + LC + OSCAIL  + Finch FP  &  0.5339 & 0.0310  &  0.4805 & 0.0149  &  0.4597 & 0.0114  &  0.4609 & 0.0061  \\
\hline
Energy OOD  + LC + OSCAIL  + Finch FP  &  0.4619 & 0.0134  &  0.4431 & 0.0100  &  0.4254 & 0.0113  &  0.4271 & 0.0092  \\
\hline
EVM OOD   + LC + OSCAIL  + Finch FP  &  0.5261 & 0.0205  &  0.4811 & 0.0205  &  0.4425 & 0.0099  &  0.4456 & 0.0069  \\
\hline
SM OOD  + LC + MEVM  + Finch FP  &  0.4938 & 0.0235 & 0.4666 & 0.0202 & 0.4712 & 0.0059 & 0.4860 & 0.0113  \\
\hline
Energy OOD  + LC + MEVM  + Finch FP  &  0.4493 & 0.0141 & 0.4408 & 0.0101 & 0.4384 & 0.0095 & 0.4421 & 0.0055  \\
\hline
EVM OOD   + LC + MEVM  + Finch FP  &  0.4781 & 0.0216 & 0.4599 & 0.0188 & 0.4554 & 0.0077 & 0.4672 & 0.0007  \\
\hline
SM OOD  + MEVM    + MEVM  + Finch FP  &  0.4556 & 0.0269 & 0.4260 & 0.0252 & 0.4319 & 0.0074 & 0.4478 & 0.0119  \\
\hline
Energy OOD  + MEVM    + MEVM  + Finch FP  &  0.4129 & 0.0156 & 0.4032 & 0.0144 & 0.4008 & 0.0089 & 0.4045 & 0.0068  \\
\hline
EVM OOD  + MEVM    + MEVM  + Finch FP  &  0.4427 & 0.0239 & 0.4234 & 0.0245 & 0.4195 & 0.0037 & 0.4326 & 0.0034  \\
\hline
Full MEVM  + Finch FP  &  0.4614 & 0.0354 & 0.4376 & 0.0231  &  0.4443 & 0.0072 & 0.4595 & 0.0086  \\
\hline
\end{tabular}
\end{center}}
\label{table_compare_baselines_2}
\end{table*}

\begin{table*}[!t]
\caption{Median on 5 tests, open-world scores of last 1000 images when the algorithm is TOWL-FEVM. S: Supervised features, I: MoCo V2 on ImageNet 2012 features, P: MoCo V2 on Places 365-standard features, +: concatenation, wL: with label, i.e., supervised open-world learning, SVM 1 (5): a SVM traind on validation with 1 (5) cluster(s) as positive and rest are negative.}
{\small
\begin{center}
\begin{tabular}{|c|c|c|c|c|c|c|c|c|} 
\hline 
\# Unknown classes & Feature Extractor & S & I & P & IP & SI & SP & SIP  \\ 
\hline \hline
\multirow{5}{*}{5} & No adaption (control case) &  0.4593 & 0.1822 & 0.1396 & 0.1775 & 0.4779 & 0.4766 & 0.4760   \\ 
\cline{2-9}
 & With label (upper bound) & 0.6958 & 0.4161 & 0.3335 & 0.4063 &  0.7109 & 0.7219 & 0.7145   \\ 
\cline{2-9}
 & OWL without SVM & 0.5364 & 0.2340 & 0.1670 & 0.2272 & 0.5592 & 0.5534 & 0.5510  \\ 
\cline{2-9}
 & OWL with SVM 1 & 0.5631 & 0.2414 & 0.1692 & 0.2344 & 0.5854 & 0.5969 & 0.5620  \\ 
\cline{2-9}
 & OWL with SVM 5 & 0.5542 & 0.2384 & 0.1779 & 0.2407 & 0.5912 & 0.5734 & 0.5685  \\ 
\hline \hline \hline
\multirow{5}{*}{10} & No adaption (control case) &  0.3759 & 0.1517 & 0.1036 & 0.1463 & 0.3906 & 0.3872 & 0.3882  \\ 
\cline{2-9}
 & With label (upper bound) & 0.6279 & 0.3484 & 0.2745 & 0.3328 & 0.6457 & 0.6431 & 0.6411  \\ 
\cline{2-9}
 & OWL without SVM & 0.4734 & 0.2071 & 0.1531 & 0.2080 & 0.5299 & 0.5058 & 0.5269  \\ 
\cline{2-9}
 & OWL with SVM 1 & 0.4858 & 0.2244 & 0.1543 & 0.2208 & 0.5013 & 0.5581 & 0.4546  \\ 
\cline{2-9}
 & OWL with SVM 5 & 0.5317 & 0.2168 & 0.1666 & 0.2165 & 0.5690 & 0.5556 & 0.5511  \\ 
\hline \hline \hline
\multirow{5}{*}{25} & No adaption (control case) &  0.3243 & 0.1353 & 0.1052 & 0.1332 & 0.3417 & 0.3367 & 0.3310   \\ 
\cline{2-9}
 & With label (upper bound) & 0.5656 & 0.2550 & 0.2029 & 0.2560 & 0.5788 & 0.5800 & 0.5759  \\ 
\cline{2-9}
 & OWL without SVM & 0.4140 & 0.1885 & 0.1463 & 0.1778 & 0.4203 & 0.4243 & 0.4228  \\ 
\cline{2-9}
 & OWL with SVM 1 & 0.3738 & 0.1615 & 0.1367 & 0.1810 & 0.4150 & 0.4650 & 0.3644  \\ 
\cline{2-9}
 & OWL with SVM 5 & 0.4842 & 0.1831 & 0.1427 & 0.1829 & 0.4895 & 0.4957 & 0.5034  \\ 
\hline \hline \hline
\multirow{5}{*}{50} & No adaption (control case) &  0.3090 & 0.1254 & 0.08851 & 0.1184 & 0.3288 & 0.3299 & 0.3318   \\ 
\cline{2-9}
 & With label (upper bound) & 0.5033 & 0.2142 & 0.1735 & 0.2080 & 0.5214 & 0.5245 & 0.5156  \\ 
\cline{2-9}
 & OWL without SVM & 0.3907 & 0.1668 & 0.1206 & 0.1585 & 0.4082 & 0.4043 & 0.4024  \\ 
\cline{2-9}
 & OWL with SVM 1 & 0.3238 & 0.1445 & 0.0976 & 0.1511 & 0.3647 & 0.4032 & 0.3354  \\ 
\cline{2-9}
 & OWL with SVM 5 & 0.4213 & 0.1629 & 0.1114 & 0.1526 & 0.4377 & 0.4533 & 0.4529  \\ 
\hline 
\end{tabular}
\end{center}}
\label{table_open_world_score_1000}
\end{table*}

\section{Evaluation Framework and Data}
\label{sec_evaluation_framework}

Prior evaluations of open-world learning in \cite{bendale2015towards,rudd2017extreme}, were fundamentally flawed for two reasons.   In both cases the system was provided with all labels for all new data, hence while they could detect unknowns, they did not have the potential cascade of error because a point incorrectly classified as known was not labeled.  Because they did not combine the unknown detection with the labeling step we do not consider it true open-world learning. Second point is that in \cite{rudd2017extreme} they used feature extractors that were trained on ImageNet 2012 \cite{krizhevsky2012imagenet}, but then they artificially defined subsets of the 1000 classes as the base of knowns and incrementally tried to detect other ImageNet 2012 classes as the unknowns. Thus, their feature space was trained using the "unknowns" as known and hence not a meaningful framework for proper open-world evaluation, even in a supervised setting. Therefore, we require a new evaluation framework, even for supervised open-world learning agents, and we do not reproduce data/tables from those prior works because of their flaws. 

To evaluate and compare the performance of an open-world learning algorithm in the task of image classification, (1) we use all 1000 classes of the ImageNet 2012 \cite{krizhevsky2012imagenet} training data set for training agents, (2) we use combinations of the test data set of ImageNet V2 (known classes) and 166 classes of ImageNet 2010 training data set that do not overlap with ImageNet 2012 (unknown classes) \cite{jafarzadeh2021automatic}. 

We define four levels of tests: varying the number of instances per class and the number of unknown classes. Each test consists of 50 batches, where the batch size was 100 images randomly sampled from the  2500 known images and 2500 unknown images for that test. Test U5 has 5 unknown classes, where each unknown class has 500 images. Test U10 has 10 unknown classes, where each unknown class has 250 images. Test U25 uses 25 unknown classes, where each unknown class has 100 images. Finally, Test U50 uses 50 unknown classes, where each unknown class has 50 images. Known classes, unknown classes, and images in each class are selected randomly. All images, known and unknown, are distributed randomly across each test.   We run each test five times and report the average and standard deviation of OWM.

\section{Experimental Results}

To get results for TOWL in Tables ~\ref{table_compare_baselines_1} -    \ref{table_open_world_score_1000}, we trained three EfficientNet-B3 networks: (1) supervised learning on all 1000 classes of ImageNet 2012 using supervised cross-entropy loss, (2) unsupervised learning on all 1000 classes of ImageNet 2012 data set \cite{krizhevsky2012imagenet} using MoCO V2 \cite{he2020momentum}, and (3) unsupervised learning on all classes of Places training data set\cite{zhou2017places} using MoCO V2 \cite{he2020momentum}. We used EfficientNet-B3 \cite{tan2019efficientnet} from Timm library \cite{timm}. Then we extracted features from the latent layer, i.e., the  layer before Logit, and froze them.

Table ~\ref{table_compare_baselines_1} compares the seven proposed fully managed baselines with other combinations of incremental learning. By comparing Table \ref{table_compare_baselines_1} with Tables \ref{table_compare_baselines_2}, \ref{table_compare_baselines_3}. our TOWL outperforms other variants without novelty management when the number of unknown classes is 5, 10, and 25. The feature that is used in the table is the concatenation of frozen feature extractors that were trained on ImageNet 2012  (supervised) and Places 365-standard  (using MoCo v2).  The best result for each group of unknowns shows in bold or italics.  Using a two-sided paired t-test for each batch/size,  the bold results are statistically significant with $p < 0.05$.  Comparing between  Tables \ref{table_compare_baselines_1} and  \ref{table_compare_baselines_2}, and the same type of test,   the improvements using management for every algorithm was statistically significant with $p < 0.01$.

Table ~\ref{table_open_world_score_1000} compares the performances of the various concatenation of self-supervised and supervised deep features and report the best performing feature fusion variant. TOWL-FEVM when using fused features, concatenation of frozen feature extractors were trained on ImageNet 2012  (supervised, cross-entropy) and Places 365-standard  (unsupervised, MoCo v2), outperforms TOWL-FEVM when it used just supervised feature extractors were trained on ImageNet 2012.  All of the results in the paper show that TOWL-FEVM, when using fused features, is superior to using either just supervised features or using just unsupervised features.

\begin{figure}[!t]
\centering
\includegraphics[width=0.7\linewidth]{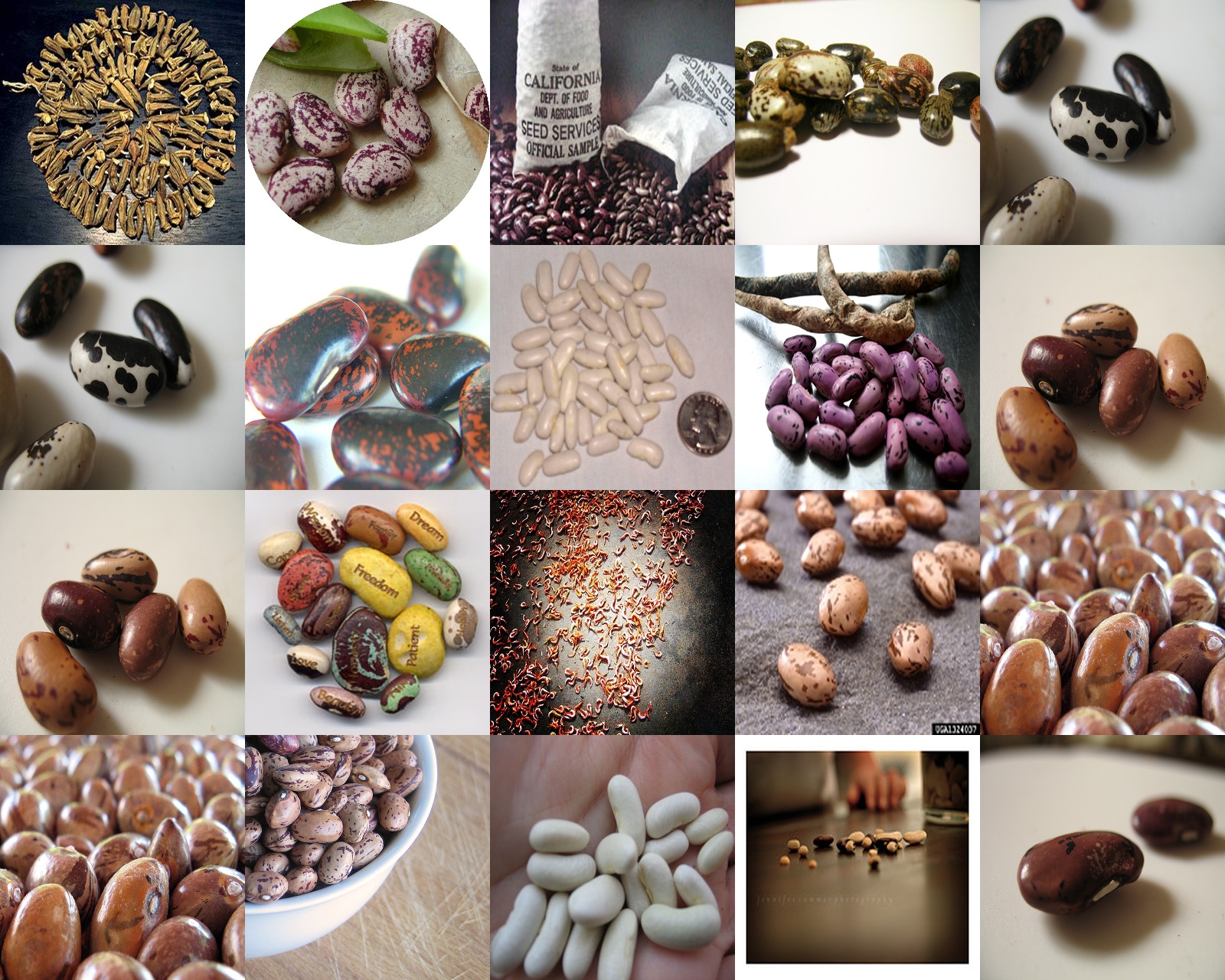}\\

(a) A good cluster not using SVM quality measure\\

\includegraphics[width=0.7\linewidth]{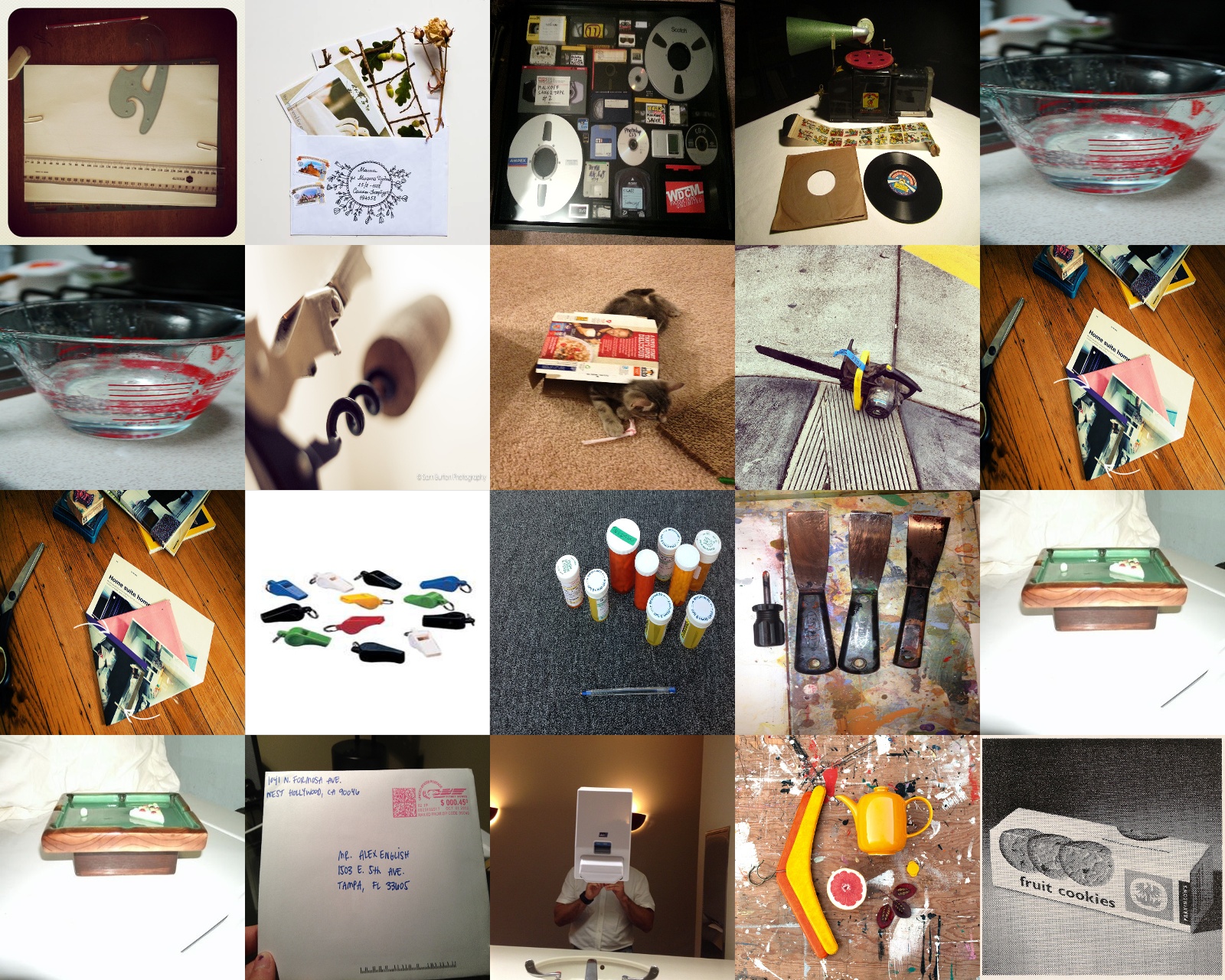}\\

(b) A bad cluster not using SVM quality measure \\

\includegraphics[width=0.7\linewidth]{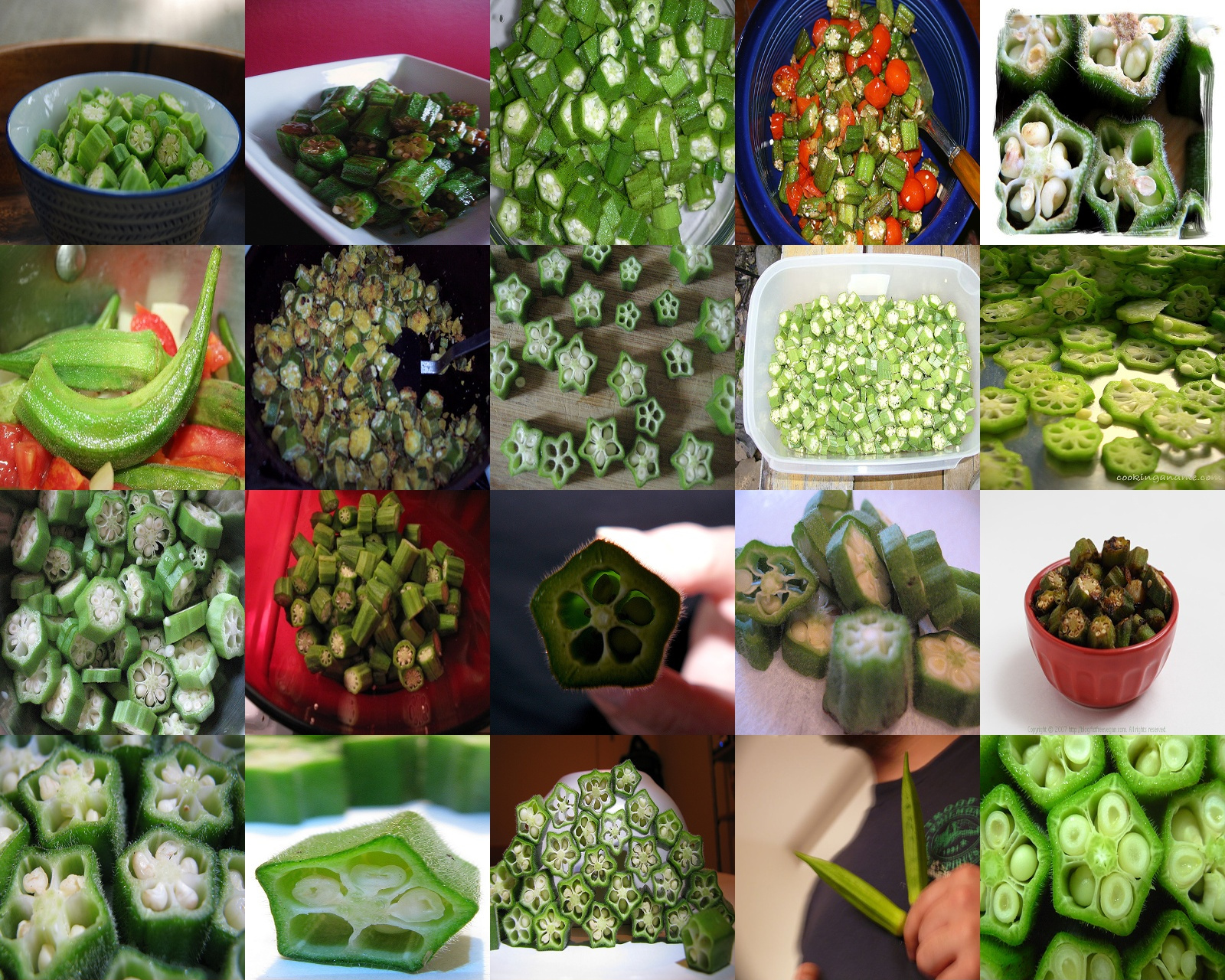}\\

(c) A good cluster using SVM quality measure\\

\includegraphics[width=0.7\linewidth]{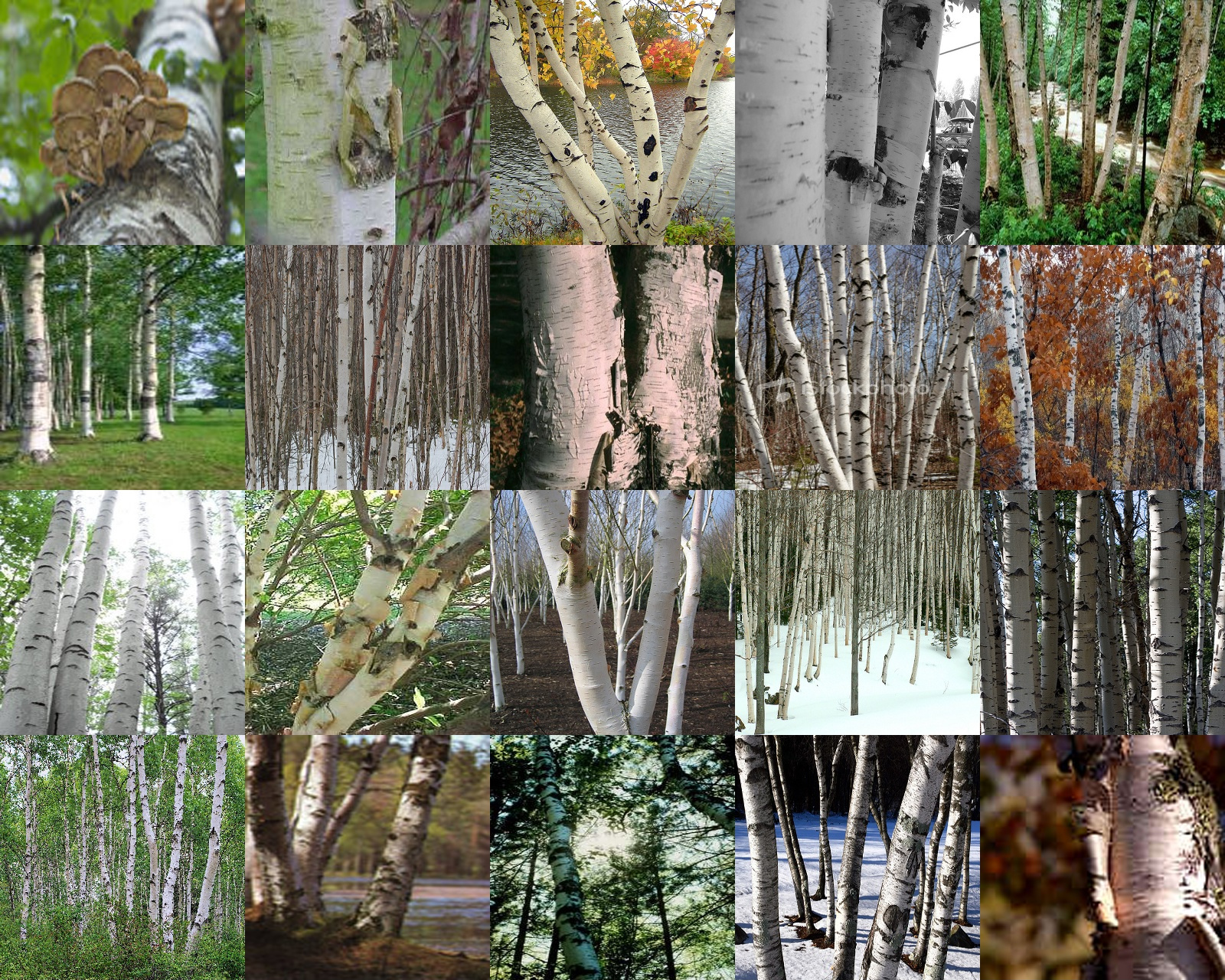}\\

(d)  A bad cluster even using SVM quality measure\\

\caption{Examples of both bad and good clusters without and with suing SVM quality measure for rejection}
\label{fig_cluster} 
\end{figure}

\section{Discussion}

Novelty discovery is a critical component of open-world learning without labels.  Producing too few clusters is dangerous as it will cause two classes to merge, and once merged, the current approach cannot separate; thus, the confusion is permanent. Over clustering will cause the new class not to generalize to the original class's full semantic concept. Therefore, selecting a partition with the proper number of clusters is necessary. Because the number of unknown classes in each batch is not constant, algorithms such as K-Means are not proper. Algorithms such as adaptive K-Means \cite{bhatia2004adaptive} and AutoNovel \cite{han2019autonovel} are slow and should be avoided.  In this paper, we used the recent state-of-the-art Finch clustering algorithm \cite{sarfraz2019efficient} to discover the class of predicted instances as unknown by MEVM. The Finch generates several partitions.   Finch optimized the cluster size in each batch independently. Our tests used the First Partition (FP) Finch partition with the maximum number of clusters because the threshold 50 to clustering was small. We also tested all algorithms with the second partition and used a two-sided paired t-test over all runs and found no statistical difference. The smallest p-value was 0.20 well above the 0.05 needed for significance. Future work should evaluate this choice with other partitions and algorithms and ideally develop a fully automatic algorithm. We encourage future researchers who find state-of-the-art real-time novelty discovery algorithms or clustering to replace Finch with those real-time algorithms.

In this paper, we used MEVM with a distance multiplayer of 0.45. Also, we set the scale of OGMM to 10, the scale of ONNO to 2, and the distance threshold of OCBCL to 10. These parameters were fixed and not varied during testing. While these values were good for EfficientNet-B3 to be evaluated on ImageNet, they should be re-evaluated on validation data for other data sets. Another parameter is how many detected novel class points are needed to begin clustering. Here, we chose a threshold of 250 to start clustering. If we chose a higher threshold, the quality of clusters will increase, and the learning speed will decrease. Therefore, there is a trade-off between the quality of learning and the speed of learning. The minimum value for these thresholds depends on the clustering algorithm and the quality of the feature extractor. We chose a threshold of four clusters to start checking with the SVM measure. We chose the threshold of 20 new points in a cluster to instantiate a new class, so few, but not one-shot unsupervised learning. If the number of required samples was larger, the new class will be better defined and generalized better; however, again, the learning speed decreases. Thus, there is a trade-off between the quality of learning and the speed of learning.

Fig. ~\ref{fig_cluster} demonstrates four clusters generated during the tests, two with and two without using SVM quality measures in the novelty management. The algorithms generated a few good clusters regardless of the presence of the SVM. When we did not use SVM, some garbage clusters passed novelty management because even though their points were far from each other, they still were the first nearest neighbors of each other and hence clustered with Finch. Even with  SVM, some clusters were not good when considering the ground-truth labels. For example, one cluster created and learned by the TOWL is tree trunks and another is tree tops. The labels in these clusters do not match with the human label in ImageNet 2010 (unknown data set), where the trees were labeled by genus/species.  These clusters caused a loss of score in the reported Tables. Therefore, the proposed metric by itself is not sufficient for evaluation. In future work and or operational systems, human labeling should be used to validate the results of the comparison.

\section{Conclusion}

In open-world learning, an autonomous agent discovers, manages novelty, and learns new classes from a non-stationary stream of data. Open-world learning is fundamentally different from open-set recognition, incremental learning, generalized novelty discovery, and generalized zero-shot learning. In an open-world environment, labels for data are often unavailable, and hence supervised open-world learning is not a scalable solution for online or real-time applications.

Here we formalized the unsupervised open-world learning problem and surveyed related work and potential components from the literature. We created a framework to evaluate autonomous agent performance in open-world scenarios and a new open-world metric suited for the evaluation of unsupervised open-world learning.  We showed that open-world learning needs more than just detection, clustering, and incrementally learning unknowns. We extended six prior incremental learning algorithms to solve open-world learning problems. The six algorithms are the Open-world Nearest Classifier Mean (ONCM), Open-world Nearest Non Outlier (ONNO),Open-world Gaussian Mixture Model (OGMM),Open-world Centroid-Based Concept Learning (OCBCL),Open-world SCAling Incremental Learning (OSCAIL), and Modified Extreme Value Machine (MEVM). Then we extended them further to build seven new True Open-World Learners (TOWL), autonomous agents that discover, characterize, manage, and learn new classes without labels from an open-world stream of data. TOWL provides  seven  baselines to unsupervised open-world learning. We found that our TOWL produced the best results for unsupervised open-world learning when it used the combination of supervised and unsupervised trained features. Finally, we illustrated that the proposed metric by itself is not sufficient and for precise comparison if the goal is matching human labels.

\bibliographystyle{IEEEtran}
\bibliography{owl_ref}

\begin{thebibliography}{100}
\providecommand{\url}[1]{#1}
\csname url@samestyle\endcsname
\providecommand{\newblock}{\relax}
\providecommand{\bibinfo}[2]{#2}
\providecommand{\BIBentrySTDinterwordspacing}{\spaceskip=0pt\relax}
\providecommand{\BIBentryALTinterwordstretchfactor}{4}
\providecommand{\BIBentryALTinterwordspacing}{\spaceskip=\fontdimen2\font plus
\BIBentryALTinterwordstretchfactor\fontdimen3\font minus
  \fontdimen4\font\relax}
\providecommand{\BIBforeignlanguage}[2]{{%
\expandafter\ifx\csname l@#1\endcsname\relax
\typeout{** WARNING: IEEEtran.bst: No hyphenation pattern has been}%
\typeout{** loaded for the language `#1'. Using the pattern for}%
\typeout{** the default language instead.}%
\else
\language=\csname l@#1\endcsname
\fi
#2}}
\providecommand{\BIBdecl}{\relax}
\BIBdecl

\bibitem{bendale2015towards}
A.~Bendale and T.~Boult, ``Towards open world recognition,'' in \emph{IEEE
  CVPR}, 2015, pp. 1893--1902.

\bibitem{rudd2017extreme}
E.~M. Rudd, L.~P. Jain, W.~J. Scheirer, and T.~E. Boult, ``The extreme value
  machine,'' \emph{IEEE transactions on pattern analysis and machine
  intelligence}, vol.~40, no.~3, pp. 762--768, 2017.

\bibitem{liu2019large}
Z.~Liu, Z.~Miao, X.~Zhan, J.~Wang, B.~Gong, and S.~X. Yu, ``Large-scale
  long-tailed recognition in an open-world,'' in \emph{Proceedings of the IEEE
  Conference on Computer Vision and Pattern Recognition}, 2019, pp. 2537--2546.

\bibitem{geng2020recent}
C.~Geng, S.-j. Huang, and S.~Chen, ``Recent advances in open set recognition: A
  survey,'' \emph{IEEE Transactions on Pattern Analysis and Machine
  Intelligence}, 2020.

\bibitem{boult2019learning}
T.~E. Boult, S.~Cruz, A.~R. Dhamija, M.~Gunther, J.~Henrydoss, and W.~J.
  Scheirer, ``Learning and the unknown: Surveying steps toward open world
  recognition,'' in \emph{Proceedings of the AAAI Conference on Artificial
  Intelligence}, vol.~33, 2019, pp. 9801--9807.

\bibitem{guerriero2018ncm}
S.~Guerriero, B.~Caputo, and T.~Mensink, ``Deepncm: Deep nearest class mean
  classifiers,'' in \emph{Proceedings of the International Conference on
  Learning Representations, Workshop Track}, 2018.

\bibitem{arandjelovic2005gmm}
O.~Arandjelovic and R.~Cipolla, ``Incremental learning of temporally-coherent
  gaussian mixture models,'' in \emph{BMVC 2005: Proceedings of the British
  Machine Conference 2005}.\hskip 1em plus 0.5em minus 0.4em\relax BMVA Press,
  2005, pp. 59--1.

\bibitem{ayub2020cbcl}
A.~Ayub and A.~R. Wagner, ``Cognitively-inspired model for incremental learning
  using a few examples,'' in \emph{Proceedings of the IEEE/CVF Conference on
  Computer Vision and Pattern Recognition Workshops}, 2020, pp. 222--223.

\bibitem{belouadah2020scail}
E.~Belouadah and A.~Popescu, ``Scail: Classifier weights scaling for class
  incremental learning,'' in \emph{Proceedings of the IEEE/CVF Winter
  Conference on Applications of Computer Vision}, 2020, pp. 1266--1275.

\bibitem{scheirer2012toward}
W.~J. Scheirer, A.~de~Rezende~Rocha, A.~Sapkota, and T.~E. Boult, ``Toward open
  set recognition,'' \emph{IEEE transactions on pattern analysis and machine
  intelligence}, vol.~35, no.~7, pp. 1757--1772, 2012.

\bibitem{dhamija2018reducing}
A.~R. Dhamija, M.~G{\"u}nther, and T.~Boult, ``Reducing network
  agnostophobia,'' in \emph{Advances in Neural Information Processing Systems},
  2018, pp. 9157--9168.

\bibitem{hendrycks17baseline}
D.~Hendrycks and K.~Gimpel, ``A baseline for detecting misclassified and
  out-of-distribution examples in neural networks,'' in \emph{Proceedings of
  International Conference on Learning Representations}, 2017.

\bibitem{liang2018enhancing}
S.~Liang, Y.~Li, and R.~Srikant, ``Enhancing the reliability of
  out-of-distribution image detection in neural networks,'' in \emph{6th
  International Conference on Learning Representations, ICLR 2018}, 2018.

\bibitem{vyas2018out}
A.~Vyas, N.~Jammalamadaka, X.~Zhu, D.~Das, B.~Kaul, and T.~L. Willke,
  ``Out-of-distribution detection using an ensemble of self supervised
  leave-out classifiers,'' in \emph{Proceedings of the European Conference on
  Computer Vision (ECCV)}, 2018, pp. 550--564.

\bibitem{Lee2020gradientsNN}
J.~Lee and G.~AlRegib, ``Gradients as a measure of uncertainty in neural
  networks,'' in \emph{2020 IEEE International Conference on Image Processing
  (ICIP)}.\hskip 1em plus 0.5em minus 0.4em\relax IEEE, 2020, pp. 2416--2420.

\bibitem{techapanurak2020hyperparameter}
E.~Techapanurak, M.~Suganuma, and T.~Okatani, ``Hyperparameter-free
  out-of-distribution detection using cosine similarity,'' in \emph{Proceedings
  of the Asian Conference on Computer Vision}, 2020.

\bibitem{liu2020energy}
W.~Liu, X.~Wang, J.~Owens, and Y.~Li, ``Energy-based out-of-distribution
  detection,'' in \emph{Advances in Neural Information Processing Systems},
  2020.

\bibitem{sehwag2019analyzing}
V.~Sehwag, A.~N. Bhagoji, L.~Song, C.~Sitawarin, D.~Cullina, M.~Chiang, and
  P.~Mittal, ``Analyzing the robustness of open-world machine learning,'' in
  \emph{Proceedings of the 12th ACM Workshop on Artificial Intelligence and
  Security}, 2019, pp. 105--116.

\bibitem{hsu2020generalized}
Y.-C. Hsu, Y.~Shen, H.~Jin, and Z.~Kira, ``Generalized odin: Detecting
  out-of-distribution image without learning from out-of-distribution data,''
  in \emph{Proceedings of the IEEE/CVF Conference on Computer Vision and
  Pattern Recognition}, 2020, pp. 10\,951--10\,960.

\bibitem{sastry2020detecting}
C.~S. Sastry and S.~Oore, ``Detecting out-of-distribution examples with gram
  matrices,'' in \emph{International Conference on Machine Learning}.\hskip 1em
  plus 0.5em minus 0.4em\relax PMLR, 2020, pp. 8491--8501.

\bibitem{mohseni2020self}
S.~Mohseni, M.~Pitale, J.~Yadawa, and Z.~Wang, ``Self-supervised learning for
  generalizable out-of-distribution detection,'' in \emph{Proceedings of the
  AAAI Conference on Artificial Intelligence}, 2020, pp. 5216--5223.

\bibitem{hendrycks2018deep}
D.~Hendrycks, M.~Mazeika, and T.~Dietterich, ``Deep anomaly detection with
  outlier exposure,'' in \emph{International Conference on Learning
  Representations}, 2018.

\bibitem{golan2018deep}
I.~Golan and R.~El-Yaniv, ``Deep anomaly detection using geometric
  transformations,'' in \emph{Proceedings of the 32nd International Conference
  on Neural Information Processing Systems}, 2018, pp. 9781--9791.

\bibitem{bergman2019classification}
L.~Bergman and Y.~Hoshen, ``Classification-based anomaly detection for general
  data,'' in \emph{International Conference on Learning Representations}, 2019.

\bibitem{ruff2019deep}
L.~Ruff, R.~A. Vandermeulen, N.~G{\"o}rnitz, A.~Binder, E.~M{\"u}ller, K.-R.
  M{\"u}ller, and M.~Kloft, ``Deep semi-supervised anomaly detection,'' in
  \emph{International Conference on Learning Representations}, 2019.

\bibitem{kimura2020adversarial}
D.~Kimura, S.~Chaudhury, M.~Narita, A.~Munawar, and R.~Tachibana, ``Adversarial
  discriminative attention for robust anomaly detection,'' in \emph{Proceedings
  of the IEEE/CVF Winter Conference on Applications of Computer Vision}, 2020,
  pp. 2172--2181.

\bibitem{nguyen2019anomaly}
D.~T. Nguyen, Z.~Lou, M.~Klar, and T.~Brox, ``Anomaly detection with
  multiple-hypotheses predictions,'' in \emph{International Conference on
  Machine Learning}.\hskip 1em plus 0.5em minus 0.4em\relax PMLR, 2019, pp.
  4800--4809.

\bibitem{zenati2018adversarially}
H.~Zenati, M.~Romain, C.-S. Foo, B.~Lecouat, and V.~Chandrasekhar,
  ``Adversarially learned anomaly detection,'' in \emph{2018 IEEE International
  conference on data mining (ICDM)}.\hskip 1em plus 0.5em minus 0.4em\relax
  IEEE, 2018, pp. 727--736.

\bibitem{li2021deep}
T.~Li, Z.~Wang, S.~Liu, and W.-Y. Lin, ``Deep unsupervised anomaly detection,''
  in \emph{Proceedings of the IEEE/CVF Winter Conference on Applications of
  Computer Vision}, 2021, pp. 3636--3645.

\bibitem{fan2020robust}
J.~Fan, Q.~Zhang, J.~Zhu, M.~Zhang, Z.~Yang, and H.~Cao, ``Robust deep
  auto-encoding gaussian process regression for unsupervised anomaly
  detection,'' \emph{Neurocomputing}, vol. 376, pp. 180--190, 2020.

\bibitem{bergmann2020uninformed}
P.~Bergmann, M.~Fauser, D.~Sattlegger, and C.~Steger, ``Uninformed students:
  Student-teacher anomaly detection with discriminative latent embeddings,'' in
  \emph{Proceedings of the IEEE/CVF Conference on Computer Vision and Pattern
  Recognition}, 2020, pp. 4183--4192.

\bibitem{abati2019latent}
D.~Abati, A.~Porrello, S.~Calderara, and R.~Cucchiara, ``Latent space
  autoregression for novelty detection,'' in \emph{Proceedings of the IEEE/CVF
  Conference on Computer Vision and Pattern Recognition}, 2019, pp. 481--490.

\bibitem{perera2019deep}
P.~Perera and V.~M. Patel, ``Deep transfer learning for multiple class novelty
  detection,'' in \emph{Proceedings of the IEEE/CVF Conference on Computer
  Vision and Pattern Recognition}, 2019, pp. 11\,544--11\,552.

\bibitem{oza2020utilizing}
P.~Oza and V.~M. Patel, ``Utilizing patch-level category activation patterns
  for multiple class novelty detection,'' in \emph{European Conference on
  Computer Vision}.\hskip 1em plus 0.5em minus 0.4em\relax Springer, 2020, pp.
  421--437.

\bibitem{tack2020csi}
J.~Tack, S.~Mo, J.~Jeong, and J.~Shin, ``Csi: Novelty detection via contrastive
  learning on distributionally shifted instances,'' in \emph{34th Conference on
  Neural Information Processing Systems (NeurIPS) 2020}.\hskip 1em plus 0.5em
  minus 0.4em\relax Neural Information Processing Systems, 2020.

\bibitem{zhang2020multi}
Y.~Zhang, Y.~Gong, H.~Zhu, X.~Bai, and W.~Tang, ``Multi-head enhanced
  self-attention network for novelty detection,'' \emph{Pattern Recognition},
  vol. 107, p. 107486, 2020.

\bibitem{lee2018hierarchical}
K.~Lee, K.~Lee, K.~Min, Y.~Zhang, J.~Shin, and H.~Lee, ``Hierarchical novelty
  detection for visual object recognition,'' in \emph{Proceedings of the IEEE
  Conference on Computer Vision and Pattern Recognition}, 2018, pp. 1034--1042.

\bibitem{schultheiss2017finding}
A.~Schultheiss, C.~K{\"a}ding, A.~Freytag, and J.~Denzler, ``Finding the
  unknown: Novelty detection with extreme value signatures of deep neural
  activations,'' in \emph{German Conference on Pattern Recognition}.\hskip 1em
  plus 0.5em minus 0.4em\relax Springer, 2017, pp. 226--238.

\bibitem{oza2020multiple}
P.~Oza, H.~V. Nguyen, and V.~M. Patel, ``Multiple class novelty detection under
  data distribution shift,'' in \emph{European Conference on Computer
  Vision}.\hskip 1em plus 0.5em minus 0.4em\relax Springer, 2020, pp. 432--449.

\bibitem{bhattacharjee2020multi}
S.~Bhattacharjee, D.~Mandal, and S.~Biswas, ``Multi-class novelty detection
  using mix-up technique,'' in \emph{Proceedings of the IEEE/CVF Winter
  Conference on Applications of Computer Vision}, 2020, pp. 1400--1409.

\bibitem{yoshihashi2019classification}
R.~Yoshihashi, W.~Shao, R.~Kawakami, S.~You, M.~Iida, and T.~Naemura,
  ``Classification-reconstruction learning for open-set recognition,'' in
  \emph{Proceedings of the IEEE Conference on Computer Vision and Pattern
  Recognition}, 2019, pp. 4016--4025.

\bibitem{oza2019c2ae}
P.~Oza and V.~M. Patel, ``C2ae: Class conditioned auto-encoder for open-set
  recognition,'' in \emph{Proceedings of the IEEE Conference on Computer Vision
  and Pattern Recognition}, 2019, pp. 2307--2316.

\bibitem{miller2021class}
D.~Miller, N.~Sunderhauf, M.~Milford, and F.~Dayoub, ``Class anchor clustering:
  A loss for distance-based open set recognition,'' in \emph{Proceedings of the
  IEEE/CVF Winter Conference on Applications of Computer Vision}, 2021, pp.
  3570--3578.

\bibitem{sun2020conditional}
X.~Sun, Z.~Yang, C.~Zhang, K.-V. Ling, and G.~Peng, ``Conditional gaussian
  distribution learning for open set recognition,'' in \emph{Proceedings of the
  IEEE/CVF Conference on Computer Vision and Pattern Recognition}, 2020, pp.
  13\,480--13\,489.

\bibitem{cevikalp2019polyhedral}
H.~Cevikalp and H.~Saglamlar, ``Polyhedral conic classifiers for computer
  vision applications and open set recognition,'' \emph{IEEE Transactions on
  Pattern Analysis and Machine Intelligence}, 2019.

\bibitem{cho2015unsupervised}
M.~Cho, S.~Kwak, C.~Schmid, and J.~Ponce, ``Unsupervised object discovery and
  localization in the wild: Part-based matching with bottom-up region
  proposals,'' in \emph{Proceedings of the IEEE conference on computer vision
  and pattern recognition}, 2015, pp. 1201--1210.

\bibitem{hsu2018learning}
Y.-C. Hsu, Z.~Lv, and Z.~Kira, ``Learning to cluster in order to transfer
  across domains and tasks,'' in \emph{International Conference on Learning
  Representations}, 2018.

\bibitem{han2019learning}
K.~Han, A.~Vedaldi, and A.~Zisserman, ``Learning to discover novel visual
  categories via deep transfer clustering,'' in \emph{Proceedings of the
  IEEE/CVF International Conference on Computer Vision}, 2019, pp. 8401--8409.

\bibitem{vo2019unsupervised}
H.~V. Vo, F.~Bach, M.~Cho, K.~Han, Y.~LeCun, P.~P{\'e}rez, and J.~Ponce,
  ``Unsupervised image matching and object discovery as optimization,'' in
  \emph{Proceedings of the IEEE/CVF Conference on Computer Vision and Pattern
  Recognition}, 2019, pp. 8287--8296.

\bibitem{vo2020toward}
H.~V. Vo, P.~P{\'e}rez, and J.~Ponce, ``Toward unsupervised, multi-object
  discovery in large-scale image collections,'' in \emph{European Conference on
  Computer Vision}.\hskip 1em plus 0.5em minus 0.4em\relax Springer, 2020, pp.
  779--795.

\bibitem{wei2019unsupervised}
X.-S. Wei, C.-L. Zhang, J.~Wu, C.~Shen, and Z.-H. Zhou, ``Unsupervised object
  discovery and co-localization by deep descriptor transformation,''
  \emph{Pattern Recognition}, vol.~88, pp. 113--126, 2019.

\bibitem{qing2021end}
Y.~Qing, Y.~Zeng, Q.~Cao, and G.-B. Huang, ``End-to-end novel visual categories
  learning via auxiliary self-supervision,'' \emph{Neural Networks}, 2021.

\bibitem{han2019autonovel}
K.~Han, S.-A. Rebuffi, S.~Ehrhardt, A.~Vedaldi, and A.~Zisserman,
  ``Automatically discovering and learning new visual categories with ranking
  statistics,'' in \emph{International Conference on Learning Representations},
  2019.

\bibitem{zhong2020openmix}
Z.~Zhong, L.~Zhu, Z.~Luo, S.~Li, Y.~Yang, and N.~Sebe, ``Openmix: Reviving
  known knowledge for discovering novel visual categories in an open world,''
  \emph{arXiv preprint arXiv:2004.05551}, 2020.

\bibitem{lee2020visualizing}
J.~H. Lee and K.~L. Wagstaff, ``Visualizing image content to explain novel
  image discovery,'' \emph{Data Mining and Knowledge Discovery}, vol.~34,
  no.~6, pp. 1777--1804, 2020.

\bibitem{romera2015embarrassingly}
B.~Romera-Paredes and P.~Torr, ``An embarrassingly simple approach to zero-shot
  learning,'' in \emph{International conference on machine learning}.\hskip 1em
  plus 0.5em minus 0.4em\relax PMLR, 2015, pp. 2152--2161.

\bibitem{xian2017zero}
Y.~Xian, B.~Schiele, and Z.~Akata, ``Zero-shot learning-the good, the bad and
  the ugly,'' in \emph{Proceedings of the IEEE Conference on Computer Vision
  and Pattern Recognition}, 2017, pp. 4582--4591.

\bibitem{xian2018zero}
Y.~Xian, C.~H. Lampert, B.~Schiele, and Z.~Akata, ``Zero-shot learning—a
  comprehensive evaluation of the good, the bad and the ugly,'' \emph{IEEE
  transactions on pattern analysis and machine intelligence}, vol.~41, no.~9,
  pp. 2251--2265, 2018.

\bibitem{bansal2018zero}
A.~Bansal, K.~Sikka, G.~Sharma, R.~Chellappa, and A.~Divakaran, ``Zero-shot
  object detection,'' in \emph{Proceedings of the European Conference on
  Computer Vision (ECCV)}, 2018, pp. 384--400.

\bibitem{guo2018zero}
Y.~Guo, G.~Ding, J.~Han, and S.~Tang, ``Zero-shot learning with attribute
  selection,'' in \emph{Proceedings of the AAAI Conference on Artificial
  Intelligence}, 2018.

\bibitem{li2020symmetry}
Y.-L. Li, Y.~Xu, X.~Mao, and C.~Lu, ``Symmetry and group in attribute-object
  compositions,'' in \emph{Proceedings of the IEEE/CVF Conference on Computer
  Vision and Pattern Recognition}, 2020, pp. 11\,316--11\,325.

\bibitem{rohrbach2011evaluating}
M.~Rohrbach, M.~Stark, and B.~Schiele, ``Evaluating knowledge transfer and
  zero-shot learning in a large-scale setting,'' in \emph{CVPR 2011}.\hskip 1em
  plus 0.5em minus 0.4em\relax IEEE, 2011, pp. 1641--1648.

\bibitem{socher2013zero}
R.~Socher, M.~Ganjoo, C.~D. Manning, and A.~Y. Ng, ``Zero-shot learning through
  cross-modal transfer,'' in \emph{Proceedings of the 26th International
  Conference on Neural Information Processing Systems-Volume 1}, 2013, pp.
  935--943.

\bibitem{elhoseiny2013write}
M.~Elhoseiny, B.~Saleh, and A.~Elgammal, ``Write a classifier: Zero-shot
  learning using purely textual descriptions,'' in \emph{Proceedings of the
  IEEE International Conference on Computer Vision}, 2013, pp. 2584--2591.

\bibitem{mancini2021open}
M.~Mancini, M.~Naeem, Y.~Xian, and Z.~Akata, ``Open world compositional
  zero-shot learning,'' in \emph{To appear, 34th IEEE Conference on Computer
  Vision and Pattern Recognition}.\hskip 1em plus 0.5em minus 0.4em\relax IEEE,
  2021.

\bibitem{fu2019vocabulary}
Y.~Fu, X.~Wang, H.~Dong, Y.-G. Jiang, M.~Wang, X.~Xue, and L.~Sigal,
  ``Vocabulary-informed zero-shot and open-set learning,'' \emph{IEEE
  transactions on pattern analysis and machine intelligence}, 2019.

\bibitem{gune2019generalized}
O.~Gune, A.~More, B.~Banerjee, and S.~Chaudhuri, ``Generalized zero-shot
  learning using open set recognition.'' in \emph{BMVC}, 2019, p. 213.

\bibitem{liu2018generalized}
S.~Liu, M.~Long, J.~Wang, and M.~I. Jordan, ``Generalized zero-shot learning
  with deep calibration network,'' in \emph{Advances in Neural Information
  Processing Systems}, 2018, pp. 2005--2015.

\bibitem{atzmon2019adaptive}
Y.~Atzmon and G.~Chechik, ``Adaptive confidence smoothing for generalized
  zero-shot learning,'' in \emph{Proceedings of the IEEE/CVF Conference on
  Computer Vision and Pattern Recognition}, 2019, pp. 11\,671--11\,680.

\bibitem{mandal2019out}
D.~Mandal, S.~Narayan, S.~K. Dwivedi, V.~Gupta, S.~Ahmed, F.~S. Khan, and
  L.~Shao, ``Out-of-distribution detection for generalized zero-shot action
  recognition,'' in \emph{Proceedings of the IEEE/CVF Conference on Computer
  Vision and Pattern Recognition}, 2019, pp. 9985--9993.

\bibitem{bhattacharjee2019autoencoder}
S.~Bhattacharjee, D.~Mandal, and S.~Biswas, ``Autoencoder based novelty
  detection for generalized zero shot learning,'' in \emph{2019 IEEE
  international conference on image processing (ICIP)}.\hskip 1em plus 0.5em
  minus 0.4em\relax IEEE, 2019, pp. 3646--3650.

\bibitem{geng2020guided}
C.~Geng, L.~Tao, and S.~Chen, ``Guided cnn for generalized zero-shot and
  open-set recognition using visual and semantic prototypes,'' \emph{Pattern
  Recognition}, vol. 102, p. 107263, 2020.

\bibitem{chen2020boundary}
X.~Chen, X.~Lan, F.~Sun, and N.~Zheng, ``A boundary based out-of-distribution
  classifier for generalized zero-shot learning,'' in \emph{European Conference
  on Computer Vision}.\hskip 1em plus 0.5em minus 0.4em\relax Springer, 2020,
  pp. 572--588.

\bibitem{zhang2018triple}
H.~Zhang, Y.~Long, Y.~Guan, and L.~Shao, ``Triple verification network for
  generalized zero-shot learning,'' \emph{IEEE Transactions on Image
  Processing}, vol.~28, no.~1, pp. 506--517, 2018.

\bibitem{rahman2018unified}
S.~Rahman, S.~Khan, and F.~Porikli, ``A unified approach for conventional
  zero-shot, generalized zero-shot, and few-shot learning,'' \emph{IEEE
  Transactions on Image Processing}, vol.~27, no.~11, pp. 5652--5667, 2018.

\bibitem{huynh2020shared}
D.~Huynh and E.~Elhamifar, ``A shared multi-attention framework for multi-label
  zero-shot learning,'' in \emph{Proceedings of the IEEE/CVF Conference on
  Computer Vision and Pattern Recognition}, 2020, pp. 8776--8786.

\bibitem{verma2018generalized}
V.~K. Verma, G.~Arora, A.~Mishra, and P.~Rai, ``Generalized zero-shot learning
  via synthesized examples,'' in \emph{Proceedings of the IEEE conference on
  computer vision and pattern recognition}, 2018, pp. 4281--4289.

\bibitem{kodirov2017semantic}
E.~Kodirov, T.~Xiang, and S.~Gong, ``Semantic autoencoder for zero-shot
  learning,'' in \emph{Proceedings of the IEEE conference on computer vision
  and pattern recognition}, 2017, pp. 3174--3183.

\bibitem{felix2018multi}
R.~Felix, I.~Reid, G.~Carneiro \emph{et~al.}, ``Multi-modal cycle-consistent
  generalized zero-shot learning,'' in \emph{Proceedings of the European
  Conference on Computer Vision (ECCV)}, 2018, pp. 21--37.

\bibitem{xie2020region}
G.-S. Xie, L.~Liu, F.~Zhu, F.~Zhao, Z.~Zhang, Y.~Yao, J.~Qin, and L.~Shao,
  ``Region graph embedding network for zero-shot learning,'' in \emph{European
  Conference on Computer Vision}.\hskip 1em plus 0.5em minus 0.4em\relax
  Springer, 2020, pp. 562--580.

\bibitem{shmelkov2017incremental}
K.~Shmelkov, C.~Schmid, and K.~Alahari, ``Incremental learning of object
  detectors without catastrophic forgetting,'' in \emph{Proceedings of the IEEE
  International Conference on Computer Vision}, 2017, pp. 3400--3409.

\bibitem{kemker2018fearnet}
R.~Kemker and C.~Kanan, ``Fearnet: Brain-inspired model for incremental
  learning,'' in \emph{International Conference on Learning Representations},
  2018.

\bibitem{belouadah2018deesil}
E.~Belouadah and A.~Popescu, ``Deesil: Deep-shallow incremental learning.'' in
  \emph{Proceedings of the European Conference on Computer Vision (ECCV)},
  2018, pp. 0--0.

\bibitem{belouadah2019il2m}
------, ``Il2m: Class incremental learning with dual memory,'' in
  \emph{Proceedings of the IEEE/CVF International Conference on Computer
  Vision}, 2019, pp. 583--592.

\bibitem{ayub2020storing}
A.~Ayub and A.~R. Wagner, ``Storing encoded episodes as concepts for continual
  learning,'' in \emph{ICML 2020 Workshop on Lifelong Machine Learning}, 2020.

\bibitem{ayub2021eec}
------, ``Eec: Learning to encode and regenerate images for continual
  learning,'' in \emph{International Conference on Learning Representations},
  2021.

\bibitem{guerriero18openreview}
S.~Guerriero, B.~Caputo, and T.~Mensink, ``Deepncm: Deep nearest class mean
  classifiers,'' in \emph{International Conference on Learning Representations
  - Workshop (ICLRw)}, 2018.

\bibitem{xue2017incremental}
N.~Xue, Y.~Wang, X.~Fan, and M.~Min, ``Incremental zero-shot learning based on
  attributes for image classification,'' in \emph{2017 IEEE International
  Conference on Image Processing (ICIP)}.\hskip 1em plus 0.5em minus
  0.4em\relax IEEE, 2017, pp. 850--854.

\bibitem{li2019incremental}
Y.~Li, Y.~Wang, Q.~Liu, C.~Bi, X.~Jiang, and S.~Sun, ``Incremental
  semi-supervised learning on streaming data,'' \emph{Pattern Recognition},
  vol.~88, pp. 383--396, 2019.

\bibitem{din2020online}
S.~U. Din, J.~Shao, J.~Kumar, W.~Ali, J.~Liu, and Y.~Ye, ``Online reliable
  semi-supervised learning on evolving data streams,'' \emph{Information
  Sciences}, vol. 525, pp. 153--171, 2020.

\bibitem{liu2020semi}
C.~Liu, Y.~Wen, and Y.~Xue, ``Semi-supervised classification of data streams
  based on adaptive density peak clustering,'' in \emph{International
  Conference on Neural Information Processing}.\hskip 1em plus 0.5em minus
  0.4em\relax Springer, 2020, pp. 639--650.

\bibitem{wu2019large}
Y.~Wu, Y.~Chen, L.~Wang, Y.~Ye, Z.~Liu, Y.~Guo, and Y.~Fu, ``Large scale
  incremental learning,'' in \emph{Proceedings of the IEEE/CVF Conference on
  Computer Vision and Pattern Recognition}, 2019, pp. 374--382.

\bibitem{rebuffi2017icarl}
S.-A. Rebuffi, A.~Kolesnikov, G.~Sperl, and C.~H. Lampert, ``icarl: Incremental
  classifier and representation learning,'' in \emph{Proceedings of the IEEE
  conference on Computer Vision and Pattern Recognition}, 2017, pp. 2001--2010.

\bibitem{lopez2017gradient}
D.~Lopez-Paz and M.~Ranzato, ``Gradient episodic memory for continual
  learning,'' in \emph{Advances in neural information processing systems},
  2017, pp. 6467--6476.

\bibitem{castro2018end}
F.~M. Castro, M.~J. Mar{\'\i}n-Jim{\'e}nez, N.~Guil, C.~Schmid, and K.~Alahari,
  ``End-to-end incremental learning,'' in \emph{Proceedings of the European
  conference on computer vision (ECCV)}, 2018, pp. 233--248.

\bibitem{hou2019learning}
S.~Hou, X.~Pan, C.~C. Loy, Z.~Wang, and D.~Lin, ``Learning a unified classifier
  incrementally via rebalancing,'' in \emph{Proceedings of the IEEE Conference
  on Computer Vision and Pattern Recognition}, 2019, pp. 831--839.

\bibitem{zhao2020maintaining}
B.~Zhao, X.~Xiao, G.~Gan, B.~Zhang, and S.-T. Xia, ``Maintaining discrimination
  and fairness in class incremental learning,'' in \emph{Proceedings of the
  IEEE/CVF Conference on Computer Vision and Pattern Recognition}, 2020, pp.
  13\,208--13\,217.

\bibitem{he2020incremental}
J.~He, R.~Mao, Z.~Shao, and F.~Zhu, ``Incremental learning in online
  scenario,'' in \emph{Proceedings of the IEEE/CVF Conference on Computer
  Vision and Pattern Recognition}, 2020, pp. 13\,926--13\,935.

\bibitem{liu2020generative}
X.~Liu, C.~Wu, M.~Menta, L.~Herranz, B.~Raducanu, A.~D. Bagdanov, S.~Jui, and
  J.~van~de Weijer, ``Generative feature replay for class-incremental
  learning,'' in \emph{Proceedings of the IEEE/CVF Conference on Computer
  Vision and Pattern Recognition Workshops}, 2020, pp. 226--227.

\bibitem{liu2020mnemonics}
Y.~Liu, Y.~Su, A.-A. Liu, B.~Schiele, and Q.~Sun, ``Mnemonics training:
  Multi-class incremental learning without forgetting,'' in \emph{Proceedings
  of the IEEE/CVF Conference on Computer Vision and Pattern Recognition}, 2020,
  pp. 12\,245--12\,254.

\bibitem{douillard2020podnet}
A.~Douillard, M.~Cord, C.~Ollion, T.~Robert, and E.~Valle, ``Podnet: Pooled
  outputs distillation for small-tasks incremental learning,'' in
  \emph{Computer vision-ECCV 2020-16th European conference, Glasgow, UK, August
  23-28, 2020, Proceedings, Part XX}, vol. 12365.\hskip 1em plus 0.5em minus
  0.4em\relax Springer, 2020, pp. 86--102.

\bibitem{douillard2020plop}
A.~Douillard, Y.~Chen, A.~Dapogny, and M.~Cord, ``Plop: Learning without
  forgetting for continual semantic segmentation,'' in \emph{Proceedings of the
  IEEE/CVF Conference on Computer Vision and Pattern Recognition}, 2021.

\bibitem{yu2020semantic}
L.~Yu, B.~Twardowski, X.~Liu, L.~Herranz, K.~Wang, Y.~Cheng, S.~Jui, and
  J.~v.~d. Weijer, ``Semantic drift compensation for class-incremental
  learning,'' in \emph{Proceedings of the IEEE/CVF Conference on Computer
  Vision and Pattern Recognition}, 2020, pp. 6982--6991.

\bibitem{zhang2020class}
J.~Zhang, J.~Zhang, S.~Ghosh, D.~Li, S.~Tasci, L.~Heck, H.~Zhang, and C.-C.~J.
  Kuo, ``Class-incremental learning via deep model consolidation,'' in
  \emph{Proceedings of the IEEE/CVF Winter Conference on Applications of
  Computer Vision}, 2020, pp. 1131--1140.

\bibitem{mi2020generalized}
F.~Mi, L.~Kong, T.~Lin, K.~Yu, and B.~Faltings, ``Generalized class incremental
  learning,'' in \emph{Proceedings of the IEEE/CVF Conference on Computer
  Vision and Pattern Recognition Workshops}, 2020, pp. 240--241.

\bibitem{kurmi2021not}
V.~K. Kurmi, B.~N. Patro, V.~K. Subramanian, and V.~P. Namboodiri, ``Do not
  forget to attend to uncertainty while mitigating catastrophic forgetting,''
  in \emph{Proceedings of the IEEE/CVF Winter Conference on Applications of
  Computer Vision}, 2021, pp. 736--745.

\bibitem{xu2019open}
H.~Xu, B.~Liu, L.~Shu, and P.~Yu, ``Open-world learning and application to
  product classification,'' in \emph{The World Wide Web Conference}, 2019, pp.
  3413--3419.

\bibitem{dhamija2021self}
A.~R. Dhamija, T.~Ahmad, J.~Schwan, M.~Jafarzadeh, C.~Li, and T.~E. Boult,
  ``Self-supervised features improve open-world learning,'' \emph{arXiv
  preprint arXiv:2102.07848}, 2021.

\bibitem{cao2021open}
K.~Cao, M.~Brbic, and J.~Leskovec, ``Open-world semi-supervised learning,''
  \emph{arXiv preprint arXiv:2102.03526}, 2021.

\bibitem{guo2019multi}
X.~Guo, A.~Alipour-Fanid, L.~Wu, H.~Purohit, X.~Chen, K.~Zeng, and L.~Zhao,
  ``Multi-stage deep classifier cascades for open world recognition,'' in
  \emph{Proceedings of the 28th ACM International Conference on Information and
  Knowledge Management}, 2019, pp. 179--188.

\bibitem{joseph2021open}
K.~J. Joseph, S.~Khan, F.~S. Khan, and V.~N. Balasubramanian, ``Towards open
  world object detection,'' in \emph{Proceedings of the IEEE/CVF Conference on
  Computer Vision and Pattern Recognition (CVPR 2021)}, 2021.

\bibitem{jia2017incremental}
X.~Jia, A.~Khandelwal, G.~Nayak, J.~Gerber, K.~Carlson, P.~West, and V.~Kumar,
  ``Incremental dual-memory lstm in land cover prediction,'' in
  \emph{Proceedings of the 23rd ACM SIGKDD International Conference on
  Knowledge Discovery and Data Mining}, 2017, pp. 867--876.

\bibitem{feng2020transfer}
L.~Feng and C.~Zhao, ``Transfer increment for generalized zero-shot learning,''
  \emph{IEEE Transactions on Neural Networks and Learning Systems}, 2020.

\bibitem{pernici2017unsupervised}
F.~Pernici and A.~Del~Bimbo, ``Unsupervised incremental learning of deep
  descriptors from video streams,'' in \emph{2017 IEEE International Conference
  on Multimedia \& Expo Workshops (ICMEW)}.\hskip 1em plus 0.5em minus
  0.4em\relax IEEE, 2017, pp. 477--482.

\bibitem{pernici2020self}
F.~Pernici, M.~Bruni, and A.~Del~Bimbo, ``Self-supervised on-line cumulative
  learning from video streams,'' \emph{Computer Vision and Image
  Understanding}, vol. 197, p. 102983, 2020.

\bibitem{lv2018unsupervised}
J.~Lv, W.~Chen, Q.~Li, and C.~Yang, ``Unsupervised cross-dataset person
  re-identification by transfer learning of spatial-temporal patterns,'' in
  \emph{Proceedings of the IEEE Conference on Computer Vision and Pattern
  Recognition}, 2018, pp. 7948--7956.

\bibitem{kalshetti2019unsupervised}
P.~Kalshetti and P.~Chaudhuri, ``Unsupervised incremental learning for hand
  shape and pose estimation,'' in \emph{ACM SIGGRAPH 2019 Posters}.\hskip 1em
  plus 0.5em minus 0.4em\relax ACM, 2019, pp. 1--2.

\bibitem{marxer2016unsupervised}
R.~Marxer and H.~Purwins, ``Unsupervised incremental online learning and
  prediction of musical audio signals,'' \emph{IEEE/ACM Transactions on Audio,
  Speech, and Language Processing}, vol.~24, no.~5, pp. 863--874, 2016.

\bibitem{rao2019continual}
D.~Rao, F.~Visin, A.~Rusu, R.~Pascanu, Y.~W. Teh, and R.~Hadsell, ``Continual
  unsupervised representation learning,'' in \emph{Advances in Neural
  Information Processing Systems}, 2019, pp. 7647--7657.

\bibitem{ma2019unsupervised}
T.~Ma, S.~Qian, J.~Cao, G.~Xue, J.~Yu, Y.~Zhu, and M.~Li, ``An unsupervised
  incremental virtual learning method for financial fraud detection,'' in
  \emph{2019 IEEE/ACS 16th International Conference on Computer Systems and
  Applications (AICCSA)}.\hskip 1em plus 0.5em minus 0.4em\relax IEEE, 2019,
  pp. 1--6.

\bibitem{aljundi2019task}
R.~Aljundi, K.~Kelchtermans, and T.~Tuytelaars, ``Task-free continual
  learning,'' in \emph{Proceedings of the IEEE Conference on Computer Vision
  and Pattern Recognition}, 2019, pp. 11\,254--11\,263.

\bibitem{stojanov2019incremental}
S.~Stojanov, S.~Mishra, N.~A. Thai, N.~Dhanda, A.~Humayun, C.~Yu, L.~B. Smith,
  and J.~M. Rehg, ``Incremental object learning from contiguous views,'' in
  \emph{Proceedings of the IEEE Conference on Computer Vision and Pattern
  Recognition}, 2019, pp. 8777--8786.

\bibitem{allred2016unsupervised}
J.~M. Allred and K.~Roy, ``Unsupervised incremental stdp learning using forced
  firing of dormant or idle neurons,'' in \emph{2016 International Joint
  Conference on Neural Networks (IJCNN)}.\hskip 1em plus 0.5em minus
  0.4em\relax IEEE, 2016, pp. 2492--2499.

\bibitem{lesort2020continual}
T.~Lesort, V.~Lomonaco, A.~Stoian, D.~Maltoni, D.~Filliat, and
  N.~D{\'\i}az-Rodr{\'\i}guez, ``Continual learning for robotics: Definition,
  framework, learning strategies, opportunities and challenges,''
  \emph{Information Fusion}, vol.~58, pp. 52--68, 2020.

\bibitem{rolnick2019experience}
D.~Rolnick, A.~Ahuja, J.~Schwarz, T.~Lillicrap, and G.~Wayne, ``Experience
  replay for continual learning,'' in \emph{Advances in Neural Information
  Processing Systems}, 2019, pp. 350--360.

\bibitem{losing2018incremental}
V.~Losing, B.~Hammer, and H.~Wersing, ``Incremental on-line learning: A review
  and comparison of state of the art algorithms,'' \emph{Neurocomputing}, vol.
  275, pp. 1261--1274, 2018.

\bibitem{kontschieder2015deep}
P.~Kontschieder, M.~Fiterau, A.~Criminisi, and S.~Rota~Bulo, ``Deep neural
  decision forests,'' in \emph{Proceedings of the IEEE international conference
  on computer vision}, 2015, pp. 1467--1475.

\bibitem{roy2020tree}
D.~Roy, P.~Panda, and K.~Roy, ``Tree-cnn: a hierarchical deep convolutional
  neural network for incremental learning,'' \emph{Neural Networks}, vol. 121,
  pp. 148--160, 2020.

\bibitem{benavides2019svm}
D.~Benavides-Prado, ``An svm-based framework for long-term learning systems,''
  in \emph{Proceedings of the AAAI Conference on Artificial Intelligence},
  vol.~33, 2019, pp. 9915--9916.

\bibitem{sarfraz2019efficient}
S.~Sarfraz, V.~Sharma, and R.~Stiefelhagen, ``Efficient parameter-free
  clustering using first neighbor relations,'' in \emph{Proceedings of the IEEE
  Conference on Computer Vision and Pattern Recognition}, 2019, pp. 8934--8943.

\bibitem{mcinnes2017hdbscan}
L.~McInnes, J.~Healy, and S.~Astels, ``hdbscan: Hierarchical density based
  clustering,'' \emph{Journal of Open Source Software}, vol.~2, no.~11, p. 205,
  2017.

\bibitem{bandaragoda2019trajectory}
T.~Bandaragoda, D.~De~Silva, D.~Kleyko, E.~Osipov, U.~Wiklund, and
  D.~Alahakoon, ``Trajectory clustering of road traffic in urban environments
  using incremental machine learning in combination with hyperdimensional
  computing,'' in \emph{2019 IEEE intelligent transportation systems conference
  (ITSC)}.\hskip 1em plus 0.5em minus 0.4em\relax IEEE, 2019, pp. 1664--1670.

\bibitem{lopez2019incremental}
E.~Lopez-Lopez, C.~V. Regueiro, X.~M. Pardo, A.~Franco, and A.~Lumini,
  ``Incremental learning techniques within a self-updating approach for face
  verification in video-surveillance,'' in \emph{Iberian Conference on Pattern
  Recognition and Image Analysis}.\hskip 1em plus 0.5em minus 0.4em\relax
  Springer, 2019, pp. 25--37.

\bibitem{lakkaraju2017identifying}
H.~Lakkaraju, E.~Kamar, R.~Caruana, and E.~Horvitz, ``Identifying unknown
  unknowns in the open world: Representations and policies for guided
  exploration,'' in \emph{Thirty-First AAAI Conference on Artificial
  Intelligence}, 2017.

\bibitem{jafarzadeh2021automatic}
M.~Jafarzadeh, T.~Ahmad, A.~R. Dhamija, C.~Li, S.~Cruz, and T.~E. Boult,
  ``Automatic open-world reliability assessment,'' in \emph{Proceedings of the
  IEEE/CVF Winter Conference on Applications of Computer Vision}, 2021, pp.
  1984--1993.

\bibitem{coles2001introduction}
S.~Coles, J.~Bawa, L.~Trenner, and P.~Dorazio, \emph{An introduction to
  statistical modeling of extreme values}.\hskip 1em plus 0.5em minus
  0.4em\relax Springer, 2001, vol. 208.

\bibitem{beirlant2006statistics}
J.~Beirlant, Y.~Goegebeur, J.~Segers, and J.~L. Teugels, \emph{Statistics of
  extremes: theory and applications}.\hskip 1em plus 0.5em minus 0.4em\relax
  John Wiley \& Sons, 2006.

\bibitem{castillo2012extreme}
E.~Castillo, \emph{Extreme value theory in engineering}.\hskip 1em plus 0.5em
  minus 0.4em\relax Elsevier, 2012.

\bibitem{henrydoss2017incremental}
J.~Henrydoss, S.~Cruz, E.~M. Rudd, M.~Gunther, and T.~E. Boult, ``Incremental
  open set intrusion recognition using extreme value machine,'' in \emph{2017
  16th IEEE International Conference on Machine Learning and Applications
  (ICMLA)}.\hskip 1em plus 0.5em minus 0.4em\relax IEEE, 2017, pp. 1089--1093.

\bibitem{mensink2013distance}
T.~Mensink, J.~Verbeek, F.~Perronnin, and G.~Csurka, ``Distance-based image
  classification: Generalizing to new classes at near-zero cost,'' \emph{IEEE
  transactions on pattern analysis and machine intelligence}, vol.~35, no.~11,
  pp. 2624--2637, 2013.

\bibitem{amigo2009comparison}
E.~Amig{\'o}, J.~Gonzalo, J.~Artiles, and F.~Verdejo, ``A comparison of
  extrinsic clustering evaluation metrics based on formal constraints,''
  \emph{Information retrieval}, vol.~12, no.~4, pp. 461--486, 2009.

\bibitem{baldwin1998description}
B.~Baldwin, T.~Morton, A.~Bagga, J.~Baldridge, R.~Chandrasekar, A.~Dimitriadis,
  K.~Snyder, and M.~Wolska, ``Description of the upenn camp system as used for
  coreference,'' in \emph{Seventh Message Understanding Conference (MUC-7):
  Proceedings of a Conference Held in Fairfax, Virginia, April 29-May 1, 1998},
  1998.

\bibitem{krizhevsky2012imagenet}
A.~Krizhevsky, I.~Sutskever, and G.~E. Hinton, ``Imagenet classification with
  deep convolutional neural networks,'' \emph{Advances in neural information
  processing systems}, vol.~25, pp. 1097--1105, 2012.

\bibitem{he2020momentum}
K.~He, H.~Fan, Y.~Wu, S.~Xie, and R.~Girshick, ``Momentum contrast for
  unsupervised visual representation learning,'' in \emph{Proceedings of the
  IEEE/CVF Conference on Computer Vision and Pattern Recognition}, 2020, pp.
  9729--9738.

\bibitem{zhou2017places}
B.~Zhou, A.~Lapedriza, A.~Khosla, A.~Oliva, and A.~Torralba, ``Places: A 10
  million image database for scene recognition,'' \emph{IEEE Transactions on
  Pattern Analysis and Machine Intelligence}, 2017.

\bibitem{tan2019efficientnet}
M.~Tan and Q.~Le, ``Efficientnet: Rethinking model scaling for convolutional
  neural networks,'' in \emph{International Conference on Machine Learning},
  2019, pp. 6105--6114.

\bibitem{timm}
\BIBentryALTinterwordspacing
R.~Wightman, \emph{PyTorch image models}, 2020. [Online]. Available:
  \url{https://github.com/rwightman/pytorch-image-models}
\BIBentrySTDinterwordspacing

\bibitem{bhatia2004adaptive}
S.~K. Bhatia \emph{et~al.}, ``Adaptive k-means clustering.'' in \emph{FLAIRS
  conference}, 2004, pp. 695--699.

\end{thebibliography}

\setcounter{section}{0}
\setcounter{table}{0}
\setcounter{figure}{0}

\renewcommand{\thesection}{S\arabic{section}} 
\renewcommand{\thesubsection}{\thesection.\arabic{subsection}}

\makeatletter
\makeatletter \renewcommand{\fnum@figure} {\figurename~S\thefigure} 
\makeatother

 \makeatletter
 \makeatletter \renewcommand{\fnum@table}
 {\tablename~S\thetable}
 \makeatother
 
\renewcommand{\theequation}{S\arabic{equation}}

\section{Deep features generally yield bounded open-space risk}
In prior work on open-set and open-world learning, \cite{scheirer2012toward,bendale2015towards}, it required algorithms have bounded open-space risk. This is important With classifiers in unbounded spaces, but we show it is not with a "deep learning" approach and deep features.

\begin{theorem}{Deep features support bounded open-space risk}
Assume  $x$ is an input of dimension $r$ from a bounded range of inputs $[l,u]^d\subset\reals^r$. Then any deep network computing using only finite combinations of bounded activation functions, the resulting feature space $F(x):\reals^r\mapsto:\reals^d$ always supports classification with bounded open space risk.
\end{theorem}
\begin{proof}
Since a finite combination of bounded transformations is itself finite,  it follows that each dimension of the output is bounded and hence there exists some $L$ and $U$ such that $\forall x\; F(x) \in [L,U]^d$.
Given some original classifier $C(y):\reals^d \mapsto {0,\ldots,k}$,   be an open-set classifier where  class 0 is for unknowns and $1,\ldots,k$ are the  known classes.  Then to ensure we have bounded open space risk, it is sufficient to define
\begin{equation}
        C'(y) := 
        \begin{cases} 
        C'(y)) & if  y \in \in [L,U]^d.\\
        0 & otherwise. 
        \end{cases}
\end{equation}
\
\end{proof}

\section{Additional Results}

\begin{table*}[!hb]
\caption{Mean and standard deviation  on 5 tests, open-world scores of last 1000 images (10 batches). Feature is concatenation of frozen feature extractors were trained on ImageNet 2012  (supervised) and Places 365-standard (using MoCo v2).}
{\small
\begin{center}
\begin{tabular}{|c|c|c|c|c|c|c|c|c|} 
\hline 
\# Unknown classes & \multicolumn{2}{c|}{5} & \multicolumn{2}{c|}{10} & \multicolumn{2}{c|}{25} & \multicolumn{2}{c|}{50} \\
\hline 
Method & $\mu$ & $\sigma$  &  $\mu$ & $\sigma$  &  $\mu$ & $\sigma$  &  $\mu$ & $\sigma$ \\
\hline \hline
SM OOD  + LC + NCM + Finch SP  &  0.4923 & 0.0086 & 0.4479 & 0.0130 & 0.4211 & 0.0036 & 0.4178 & 0.0061  \\
\hline
Energy OOD  + LC + NCM + Finch SP  &  0.4584 & 0.0133 & 0.4272 & 0.0103 & 0.4105 & 0.0056 & 0.4127 & 0.0050  \\
\hline
EVM OOD  + LC + NCM + Finch SP  &  0.4880 & 0.0110 & 0.4389 & 0.0091 & 0.4139 & 0.0068 & 0.4108 & 0.0039  \\
\hline
SM OOD  + LC + NNO + Finch SP  &  0.5093 & 0.0086 & 0.4698 & 0.0153 & 0.4546 & 0.0055 & 0.4585 & 0.0057  \\
\hline
Energy OOD  + LC + NNO + Finch SP  &  0.4627 & 0.0125 & 0.4370 & 0.0117 & 0.4230 & 0.0073 & 0.4271 & 0.0076  \\
\hline
EVM OOD  + LC + NNO + Finch SP  &  0.4977 & 0.0100 &  0.4589 & 0.0131 & 0.4398 & 0.0075 & 0.4418 & 0.0041 \\
\hline
SM OOD  + LC + GMM + Finch SP  &  0.4891 & 0.0075 & 0.4418 & 0.0124 & 0.4195 & 0.0036 & 0.4174 & 0.0062  \\
\hline
Energy OOD  + LC + GMM + Finch SP  &  0.4579 & 0.0130 & 0.4268 & 0.0100 & 0.4095 & 0.0050 & 0.4127 & 0.0050  \\
\hline
EVM OOD  + LC + GMM + Finch SP  &  0.4860 & 0.0097 & 0.4356 & 0.0107 & 0.4112 & 0.0046 & 0.4101 & 0.0042  \\
\hline
SM OOD  + LC + CBCL  + Finch SP  &  0.5518 & 0.0215 & 0.4787 & 0.0388 & 0.4394 & 0.0056 & 0.4286 & 0.0069  \\
\hline
Energy OOD  + LC + CBCL  + Finch SP  &  0.4662 & 0.0157 & 0.4396 & 0.0145 & 0.4177 & 0.0046 & 0.4189 & 0.0063  \\
\hline
EVM OOD  + LC + CBCL  + Finch SP  & 0.5282 & 0.0098 & 0.4769 & 0.0246 & 0.4208 & 0.0072 & 0.4212 & 0.0061  \\
\hline
SM OOD  + LC + SCAIL  + Finch SP  &  0.5381 & 0.0326  &  0.4839 & 0.0142  &  0.4583 & 0.0139  &  0.4604 & 0.0076  \\
\hline
Energy OOD  + LC + SCAIL  + Finch SP  &  0.4631 & 0.0137 & 0.4388 & 0.0099  &  0.4249 & 0.0129 & 0.4292 & 0.0099  \\
\hline
EVM OOD  + LC + SCAIL  + Finch SP  &  0.5342 & 0.0151  &  0.4809 & 0.0203  &  0.4400 & 0.0112  &  0.4460 & 0.0066  \\
\hline
SM OOD  + LC + EVM  + Finch SP  &  0.4971 & 0.0213 & 0.4665 & 0.0200 & 0.4711 & 0.0060 & 0.4855 & 0.0111  \\
\hline
Energy OOD  + LC + EVM  + Finch SP  &  0.4493 & 0.0141 & 0.4408 & 0.0101 & 0.4386 & 0.0093 & 0.4421 & 0.0055  \\
\hline
EVM OOD  + LC + EVM  + Finch SP  &  0.4835 & 0.0192 & 0.4599 & 0.0192 & 0.4560 & 0.0077 & 0.4673 & 0.0007  \\
\hline
SM OOD  + EVM    + EVM  + Finch SP  &  0.4582 & 0.0262 & 0.4254 & 0.0240 & 0.4317 & 0.0078 & 0.4465 & 0.0113  \\
\hline
Energy OOD  + EVM    + EVM  + Finch SP  &  0.4119 & 0.0156 & 0.4028 & 0.0143 & 0.4010 & 0.0088 & 0.4051 & 0.0061  \\
\hline
EVM  OOD  + EVM    + EVM  + Finch SP  &  0.4481 & 0.0185 & 0.4233 & 0.0247 & 0.4195 & 0.0037 & 0.4326 & 0.0034  \\
\hline
Full EVM  + Finch SP  &  0.4717 & 0.0304 & 0.4336 & 0.0210  &  0.4437 & 0.0068 & 0.4562 & 0.0088  \\
\hline
\end{tabular}
\end{center}}
\label{table_compare_baselines_3}
\end{table*}

\end{document}